\documentclass[12pt]{amsart}

\usepackage{amsmath, amsthm, amssymb,amscd}
\usepackage{fullpage, paralist,enumerate}
\usepackage{verbatim}
\usepackage[all]{xy}
\usepackage[parfill]{parskip}
\usepackage{graphicx}

\usepackage{afterpage}

\usepackage{amsmath}
\usepackage{amsfonts}
\usepackage{mathcomp}
\usepackage{textcomp}
\usepackage{amssymb}
\usepackage{latexsym}
\usepackage{graphicx}
\usepackage{courier}
\usepackage[english]{babel}
\usepackage{mathbbol}
\usepackage{multirow}
\usepackage{latexsym}
\usepackage{epsfig}
\usepackage{subfigure}
\usepackage{dcolumn}
\usepackage{bm}
\usepackage{colordvi}
\usepackage{color}
\usepackage{booktabs}
\usepackage{stmaryrd}
\usepackage{cite}
\usepackage[linesnumbered,commentsnumbered,ruled]{algorithm2e}
\usepackage[noend]{algpseudocode}
\usepackage{epstopdf}
\input xy
\xyoption{all}
\newcommand{\commentout}[1]{}

\newtheorem{thm}{Theorem}[section]

\newtheorem{prop}[thm]{Proposition}

\newcommand{\nwc}{\newcommand*}

\nwc{\ben}{\begin{equation*}}
\nwc{\bea}{\begin{eqnarray}}
\nwc{\beq}{\begin{eqnarray}}
\nwc{\bean}{\begin{eqnarray*}}
\nwc{\beqn}{\begin{eqnarray*}}
\nwc{\beqast}{\begin{eqnarray*}}

\nwc{\eal}{\end{align}}
\nwc{\een}{\end{equation*}}
\nwc{\eea}{\end{eqnarray}}
\nwc{\eeq}{\end{eqnarray}}
\nwc{\eean}{\end{eqnarray*}}
\nwc{\eeqn}{\end{eqnarray*}}

\CompileMatrices

\theoremstyle{remark}

\nwc{\nn}{\nonumber}
\nwc{\mb}{\mathbf}
\nwc{\ml}{\mathcal}

\newcommand{\lt}{\left}
\newcommand{\rt}{\right}

\nwc{\vep}{\varepsilon}
\nwc{\ep}{\epsilon}
\nwc{\vrho}{\varrho}
\nwc{\orho}{\bar\varrho}
\nwc{\vpsi}{\varpsi}
\nwc{\lamb}{\lambda}
\nwc{\om}{\omega}
\nwc{\Om}{\Omega}
\nwc{\al}{\alpha}
\nwc{\sgn}{\mbox{\rm sgn}}

\nwc{\IA}{\mathbb{A}} 
\nwc{\bi}{\mathbf{i}}
\nwc{\ba}{\mathbf{a}}
\nwc{\bmb}{\mathbf{b}}
\nwc{\bo}{\mathbf{o}}
\nwc{\IB}{\mathbb{B}}
\nwc{\IC}{\mathbb{C}} 
\nwc{\ID}{\mathbb{D}} 
\nwc{\IM}{\mathbb{M}} 
\nwc{\IP}{\mathbb{P}} 
\nwc{\II}{\mathbb{I}} 
\nwc{\IE}{\mathbb{E}} 
\nwc{\IF}{\mathbb{F}} 
\nwc{\IG}{\mathbb{G}} 
\nwc{\IN}{\mathbb{N}} 
\nwc{\IQ}{\mathbb{Q}} 
\nwc{\IR}{\mathbb{R}} 
\nwc{\IT}{\mathbb{T}} 
\nwc{\IZ}{\mathbb{Z}} 

\nwc{\cE}{{\ml E}}
\nwc{\cP}{{\ml P}}
\nwc{\cQ}{{\ml Q}}
\nwc{\cL}{{\ml L}}
\nwc{\cX}{{\ml X}}
\nwc{\cW}{{\ml W}}
\nwc{\cZ}{{\ml Z}}
\nwc{\cR}{{\ml R}}
\nwc{\cV}{{\ml V}}
\nwc{\cT}{{\ml T}}
\nwc{\crV}{{\ml L}_{(\delta,\rho)}}
\nwc{\cC}{{\ml C}}
\nwc{\cO}{{\ml O}}
\nwc{\cA}{{\ml A}}
\nwc{\cK}{{\ml K}}
\nwc{\cB}{{\ml B}}
\nwc{\cD}{{\ml D}}
\nwc{\cF}{{\ml F}}
\nwc{\cS}{{\ml S}}
\nwc{\cM}{{\ml M}}
\nwc{\cG}{{\ml G}}
\nwc{\cH}{{\ml H}}
\nwc{\bk}{{\mb k}}
\nwc{\bn}{{\mb n}}
\nwc{\bp}{{\mb p}}
\nwc{\bz}{\mb z}
\nwc{\bl}{{\mb l}}
\nwc{\bj}{{\mb j}}
\nwc{\bs}{{\mb s}}
\nwc{\by}{\mathbf{h}}
\nwc{\bZ}{\mathbf{Z}}
\nwc{\bF}{\mathbf{F}}
\nwc{\bE}{\mathbf{E}}
\nwc{\bV}{\mathbf{V}}
\nwc{\bY}{\mathbf Y}
\nwc{\br}{\mb r}
\nwc{\pft}{\cF^{-1}_2}
\nwc{\bU}{{\mb U}}
\nwc{\bG}{{\mb G}}
\nwc{\bg}{\mathbf{g}}
\nwc{\mbf}{\mathbf{f}}
\nwc{\mbe}{\mathbf{e}}
\nwc{\be}{\mathbf{e}}
\nwc{\ind}{\operatorname{I}}
\nwc{\mbx}{\mathbf{f}}
\nwc{\bb}{\mathbf{g}}
\nwc{\xmax}{f_{\rm max}}
\nwc{\xmin}{f_{\rm min}}
\nwc{\suppx}{\hbox{\rm supp} (\mbf)}
\nwc{\cI}{\IZ^2_N}
\nwc{\chis}{{\chi^{\rm s}}}
\nwc{\chii}{{\chi^{\rm i}}}
\nwc{\pdfi}{{f^{\rm i}}}
\nwc{\pdfs}{{f^{\rm s}}}
\nwc{\pdfii}{{f_1^{\rm i}}}
\nwc{\pdfsi}{{f_1^{\rm s}}}
\nwc{\thetatil}{{\tilde\theta}}
\nwc{\red}{\color{red}}
\nwc{\blue}{\color{blue}}

\nwc{\prox}{\hbox{prox}}
\nwc{\diag}{\hbox{\rm diag}}
\nwc{\supp}{{\hbox{\rm supp}}}

\nwc{\sloc}{J_{\rm f}}
\nwc{\bu}{{\mb u}}
\nwc{\bv}{{\mb v}}
\nwc{\cU}{\mathcal{U}}
\nwc{\cN}{\mathcal{N}}
\nwc{\bN}{\mathbf{N}}
\nwc{\mbm}{\mathbf{m}}
\nwc{\bw}{\mathbf{w}}
\nwc{\bom}{\mathbf{w}}
\nwc{\bt}{\mathbf{t}}
\nwc{\z}{y}
\nwc{\cY}{\mathcal{Y}}
\nwc{\bM}{\mathbf{M}}
\nwc{\half}{{1\over 2}}
\nwc{\Sf}{S_{\rm f}}
\nwc{\Jf}{J_{\rm f}}
\nwc{\nul}{\hbox{\rm null}_\IR}
\nwc{\spanR}{\hbox{\rm span}_\IR}
\nwc{\Arg}{\hbox{\rm Arg~}}
\nwc{\fdr}{S_{\rm f}}
\nwc{\phase}[1]{\exp\lt[i\measured #1\rt]}

\nwc{\im}{{\rm i}}

\nwc{\cle}{\preccurlyeq}

\nwc{\lb}{\llbracket}
\nwc{\rb}{\rrbracket}
\nwc{\modpi}{{{\rm mod}\,2\pi}}

\begin{document}
 \title{
Blind Ptychography by Douglas-Rachford Splitting
}
 
\author{Albert Fannjiang
 \address{
Department of Mathematics, University of California, Davis, California  95616, USA. Email:  {\tt fannjiang@math.ucdavis.edu}
} \and Zheqing Zhang
\address{ Department of Mathematics, University of California, Davis, California  95616, USA. } 
}

\begin{abstract}

 Blind ptychography is the scanning version of coherent diffractive imaging 
which seeks to recover both the object and the probe simultaneously. 

Based on alternating minimization by Douglas-Rachford splitting, AMDRS is a blind ptychographic algorithm  informed by the uniqueness theory, the Poisson noise model and the stability analysis. 
Enhanced by the initialization method and the use of a randomly phased mask,  AMDRS converges globally and geometrically. 

Three boundary conditions are considered in the simulations: periodic, dark-field and bright-field boundary conditions.
The dark-field boundary condition is suited for isolated objects while the bright-field boundary condition is for non-isolated objects.
The periodic boundary condition is a mathematically convenient reference point. Depending on the availability of the boundary prior
the dark-field and 
the bright-field boundary conditions may or may not be enforced in the reconstruction.  Not surprisingly, enforcing the boundary condition 
improves the rate of convergence, sometimes in a significant way. Enforcing the bright-field condition in the reconstruction
can also remove
the linear phase ambiguity. 

\end{abstract}

\maketitle 
\section{Introduction}
Ptychography uses a localized coherent probe to illuminate different parts of a unknown extended object 
and collect multiple diffraction patterns 
 as measurement data  (Fig. \ref{fig0}). The redundant information in the overlap between adjacent illuminated spots is then exploited  to improve phase retrieval methods  \cite{Pfeiffer,Nugent}. An important feature of ptychography   \cite{DM08,probe09} is that the probe needs not be known precisely beforehand and can
be  recovered along with the unknown object.

 Recently ptychography  has been extended to the Fourier domain \cite{FPM13, Yang14}.  In Fourier ptychography,  illumination angles are scanned sequentially with a programmable array source with the diffraction pattern measured at each angle  \cite{Tian15,FPM15}. Tilted illumination samples different regions of Fourier space, as in synthetic-aperture and structured-illumination imaging.

\commentout{
Ptychography is the scanning version of coherent diffractive imaging (CDI) \cite{CDI} that 
acquires  multiple diffraction patterns through the scan of a localized illumination on an extended object (Fig. \ref{fig0}). The redundant information in the overlap between adjacent illuminated spots is then exploited  to improve phase retrieval methods  \cite{Pfeiffer,Nugent}. Ptychography originated in electron microscopy \cite{Hoppe1, Hoppe2, EM0, EM1,EM2,EM-ptych1,EM-ptych2} and has been successfully  implemented with X-ray, optical  and terahertz waves  \cite{PIE104,PIE204,DM08, probe09, Rod08, ptycho10,terahz,FROG}.

Recently ptychography  has been extended to the Fourier domain \cite{FPM13, Yang14}.  In Fourier ptychography,  illumination angles are scanned sequentially with a programmable array source with the diffraction pattern measured at each angle  \cite{Tian15,FPM15}. Tilted illumination samples different regions of Fourier space, as in synthetic-aperture and structured-illumination imaging.

An important feature of ptychography developed in Thibault {\em et al.}  \cite{DM08,probe09} is simultaneous recovery of the object and the probe. This capability allows application of ptychography as a
beam characterization and wave front aberration sensing technique. 

}

\commentout{
 For the standard phase retrieval with a plain (uncoded) diffraction pattern, Bates \cite{Bates} and Hayes \cite{Hayes} (see also \cite{BS}) have proven that the phase problem is uniquely solvable in more than one dimension. But this theory  \cite{Bates,Hayes,BS}  comes with several caveats that degrade numerical reconstruction in practice. 
First, the uniqueness theory concerns   {\em generic} objects, i.e.
almost all objects from a random ensemble. In real life, most objects are arguably not ``generic".  Second, the theory allows the following ambiguities besides the global, constant phase factor:  translation within the field of view and  the so-called twin image, related to the true solution via complex conjugation and rotation by $180^\circ$.  

The translation and twin-image ambiguities can be removed and the uniqueness result extended from generic to every non-line object
if we add a set of coded diffraction pattern with a randomly phased  probe \cite{pum}. Moreover, one only needs a slight prior knowledge
about the range of the probe phase.  In other words, one can perform essentially {\em blind} phase retrieval with a plain and a coded
diffraction pattern to recover both the unknown object and the roughly known probe. 
}

Yet, despite significant progress that allows for reliable practical implementation, some of the technique's  fundamentals remain poorly understood \cite{ptych-theory}. 
For example, precise conditions for uniqueness of blind-ptychographic solution are not known until recently. 
Roughly speaking  the uniqueness theory \cite{blind-ptych} says  that  with the use of a randomly phased probe
and  
 for a general class of irregular scan schemes (e.g. \eqref{rank1} and \eqref{rank2}), called the mixing schemes, 
 the only ambiguities are a scaling factor and  an affine phase ramp, both of which
 are intrinsic to blind ptychographic reconstruction. 
 In contrast,  in standard, non-ptychographic phase retrieval the exit wave (the multiplication of the probe and the object) alone suffers  
 ambiguities such as a global phase factor, a global spatial translation and the twin image. Further the exit wave in standard phase retrieval can not be  unambiguously splitted into the probe and the object without other strong prior constraint \cite{pum}. This makes clear the fundamental advantage of the ptychographic method.

 The purpose of this work is to  present reconstruction schemes informed by the uniqueness theory, analyze
the algorithmic structure and demonstrate by numerical experiments the global and geometrical convergence properties of the schemes.

  \begin{figure}[t]
\begin{center}
\includegraphics[width=10cm]{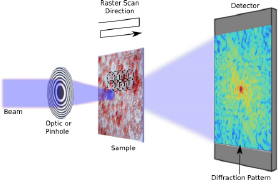}
\caption{Simplified ptychographic setup showing a Cartesian grid used for the overlapping raster scan positions \cite{parallel}.
}
\label{fig0}
\end{center}
\end{figure}

We adopt the following notations in the paper.  Let $\cT$ be the set of all shifts, including $(0,0)$,  involved  in the ptychographic measurement. 
 For general ptychographic schemes, we
 denote by $\mu^\bt$ the $\bt$-shifted probe for all $\bt\in \cT$ and $\cM^\bt$ the domain of
$\mu^\bt$. Let $f^\bt$ the object restricted to $\cM^\bt$ and  $\mb{\rm Twin}(f^\bt)$ the twin image of $f^\bt$ defined in $\cM^\bt$. 
We say that  the object part $f^\bt$ has a tight support in $\cM^\bt$ if $\cM^\bt$ does not fully contain any spatially shifted support
of $f^\bt$. Let $f:=\cup_\bt f^\bt\subseteq \cM:=\cup_\bt \cM^\bt$.

\commentout{
We say that $\{f^\bt: \bt\in \cT\}$ is strongly connected if any two parts are joined by a chain in which two successive parts $f^\bt, f^{\bt'}$ overlap 
 in the following sense: \beq
\label{s-conn}
\min\lt\{|\cM^\bt\cap \cM^{\bt'}\cap\supp(f^\bt)|, |\cM^\bt\cap \cM^{\bt'}\cap \supp(\mb{\rm Twin}(f^\bt))|\rt\}\gg 1,
\eeq
where $|\cdot|$ denotes the cardinality of the set. 
}

Several ideas can be drawn from the uniqueness theory \cite{blind-ptych} that prescribes precise conditions under which
ptychographic ambiguities are limited to a scaling factor and an affine  phase ramp. 

First the theory recommends the use of irregular perturbations of raster scan of step size $\tau$ 
which have been widely practiced in the literature. One is
\beq
\label{rank1}
\mbox{\rm Rank-one perturbation}\quad \bt_{kl}=\tau(k,l)+(\delta^1_k, \delta^2_l),\quad k,l\in \IZ
\eeq
where $\delta_{k}^1$ and $\delta^2_l$ are small random variables relative to $\tau$. 
The other is 
\beq
\label{rank2}
\mbox{\rm Full-rank perturbation}\quad \bt_{kl}=\tau (k,l)+(\delta^1_{kl},\delta^2_{kl}), \quad k,l\in \IZ
\eeq
 where $\delta_{kl}^1$ and $\delta^2_{kl}$ are small random variables relative to $\tau$. 

Second the theory  recommends the use of
a randomly phased probe with the unknown transmission function  
$
\mu^0(\bn)=|\mu^0|(\bn)e^{\im \theta(\bn)}
$
where  $\theta(\bn)$ are random variables
and  $|\mu^0|(\bn)\neq 0,\forall \bn\in \cM^0$.  Randomly phased probes have been adopted in ptychographic experiments
\cite{zone-plate,ptycho-rpi,random-aperture,diffuser}. 

Third, the theory suggests  the probe phase constraint (PPC) as the probe initialization.  
{\em 
We say  that a probe estimate $\nu^0$ satisfies PPC with $\delta<1$ if 
\beq\label{PPC}
\measuredangle (\nu^0(\bn),\mu^0(\bn))<\delta\pi,\quad \forall\bn. 
\eeq
The default case is $\delta=1/2$ and PPC has the intuitive meaning 
$\Re(\bar \nu^0\odot \mu^0)>0$ at every pixel (where $\odot$ denotes the component-wise product, the bar the complex conjugate and $\Re$ the real part), which implies certain similarity of $\nu^0$ to $\mu^0$.
}

 The constraint \eqref{PPC} represents prior information about the probe and,  in practice, needs only to hold for sufficiently large number of pixels $\bn$. We use \eqref{PPC} for selecting and quantifying initialization, instead of the usual 2-norm.  
 
Fourth, the theory advises a probe size and overlap such that  the objects parts $\{f^\bt: \bt\in \cT\}$ mutually overlap sufficiently and at least one part $f^\bt$ has a tight support in $\cM^\bt$. The reader is referred to \cite{blind-ptych}
for the technical definition of sufficient overlap which depends not only on the overlap of adjacent illuminations
but also on the object support. For example,  the overlap requirement of adjacent illuminations for a sparsely supported object is higher than that for a densely support object. In the literature,  the adjacent illuminations  have  at least 50\%, typically 60-70\%,  overlap \cite{overlap,ePIE09}. This rule of thumb is more stringent than necessary in the case of 
 an extended, densely supported object. 
 
 \subsection{Our contribution}

Our reconstruction scheme is based on alternating minimization with the Douglas-Rachford splitting (DRS) method. DRS is a proximal point method closely related to the Alternating Direction Method of Multipliers (ADMM). The existing theory and analysis for DRS and ADMM  (see, e.g. \cite{LM79,EB92,GM75,GM76}) are limited to the convex optimization, except for a few exceptions (see e.g. \cite{FDR, LP,Hesse}).  
In particular, the performance of DRS and ADMM in the non-convex setting depends sensitively on
the choice of the objective functions as well as the selection of the relaxation parameters and the step size which are the main focus of our proposed scheme. 

Our choices of the objective functions and the step size are informed by the uniqueness theory \cite{blind-ptych}, the Poisson noise model and the stability analysis. The confluence of the three considerations
leads to the proposed scheme (acronym AMDRS). The crux of our stability analysis (Proposition \ref{thm:stable} and \ref{thm:stable2}) is that  the true solution (modulo the scaling factor and the affine phase ambiguity) to blind ptychography are stable fixed points of AMDRS with any step size and, moreover,   
in the case of  the unit step size,  all other  fixed points are saddle points. 
This suggests  that a unit step size is a good choice for AMDRS in contrast to 
(sufficiently) small step sizes typically required for nonconvex convergence analysis
\cite{LP, Hesse}.

For any iterative scheme for non-convex optimization, the initialization of the iteration plays a crucial role
in the convergence of the scheme. The initialization step is often glossed over in the development of
numerical schemes due to lack of a good metric  (see e.g. \cite{rPIE17,Waller15,probe09}). In our numerical experiments, PPC turns out to be
an excellent initialization method and metric for controlling the convergence behavior of the schemes. 
In particular, our schemes converge globally and geometrically under PPC with $\delta=1/2$.

In additional to demonstrate the numerical performance of the proposed schemes, we propose
a special boundary condition (the bright-field boundary condition) to remove the linear phase ambiguity.
The capability of removing the linear phase ambiguity is particularly important for 3D blind tomography
as the different linear phase ramps for different projections would collectively create enormous 3D ambiguities 
that are difficult to make consistent. 


 
The rest of the paper is organized as follows. In Section \ref{sec:DRS}, we introduce the Douglas-Rachford splitting method as the key ingredient of our reconstruction algorithms, Gaussian-DRS and Poisson-DRS. We give the fixed point analysis in Section \ref{sec:fixed}. In Section \ref{sec:num}, we perform numerical experiments with our schemes.  In Appendix \ref{sec:likelihood}, we show that Gaussian-DRS is an asymptotic form of Poisson-DRS. We conclude  in Section \ref{sec:last}. 

\section{Alternating minimization by Douglas-Rachford splitting}\label{sec:DRS}

For simplicity, let $\supp(f)\subseteq \IZ_n^2\subseteq \cM$ and let  $\cM^0:=\IZ_m^2, m<n,$ be the initial probe region with the structured illumination given by $\mu^0$.  The pixels
in $\cM\setminus \IZ_n^2$ merits a separate treatment (referred to as boundary condition)
depending on the experimental set-up (see Section \ref{sec:BC}). 

Let $\cF(\nu,g)\in \IC^N$ be the totality of the Fourier (magnitude and phase) data  corresponding to the probe $\nu$ and the object $g$ such that $|\cF(\mu, f)|=b$ where $b$ is the noiseless ptychographic data. Since $\cF(\cdot,\cdot)$ is a bilinear function, $A_k h:=\cF(\mu_k, h), \, k\geq 1, $ defines a matrix $A_k$  for the $k$-th probe estimate $\mu_k$ and $B_k\eta :=\cF(\eta, f_{k+1}), \ k\geq 1,$  for  the $(k+1)$-st image estimate $f_{k+1}$ such that $A_k f_{j+1}=B_j\mu_k$, $j\geq 1, k\geq 1$. 

 For any $y\in \IC^{m\times m}$, $\sgn(y)$ is defined as 
\[
\sgn(y)[j]=\lt\{\begin{matrix} 
1&\mbox{if $y[j]=0$}\\
y[j]/|y[j]|&\mbox{else. }
\end{matrix}\rt.
\]

The basic strategy for blind ptychographic reconstruction is to alternately update the object and probe estimates
starting from an initial guess \cite{DM08,probe09,pum}.  We pursue the same strategy and perform the object and probe updates
by solving certain minimization problems. 

\subsection{Alternating minimization} 
\begin{algorithm}
\caption{Alternating minimization (AM)}\label{alg: suedo algorithm}
\begin{algorithmic}[1]
\State Input: initial probe guess $\mathbf{\mu}_1\in $ PPC($\bk,\delta$). 
\State Update the object estimate  
$\quad
f_{k+1}=\arg\min\cL(A_kg)$ s.t. $g\in \IC^{n\times n}$.
\State  Update the probe estimate  
$\quad
\mu_{k+1}=\arg\min\cL(B_k\nu)$ s.t.
$\nu\in \IC^{m\times m}$. 
\State Terminate if $\||B_k\mu_{k+1}|- b\|_2 $ stagnates or is less than tolerance; otherwise, go back to step 2 with $k\rightarrow k+1.$
\end{algorithmic}
\end{algorithm}
A main feature of alternating minimization scheme is  the monotonicity property:
\[
\cL(A_kf_k)\ge \cL(A_{k+1}f_{k+1}),\quad \cL(B_k\mu_k)\ge \cL(B_{k+1}\mu_{k+1}). 
\]

We consider two log-likelihood cost functions \cite{ML12,Poisson2} for
noise-robustness
\beq
\mbox{\rm Poisson:}\quad \cL(y) &=&\sum_i |y[i]|^2-b^2[i]\ln |y[i]|^2\label{Poisson}\\
\mbox{\rm Gaussian:}\quad  \cL(y)&=&\half \| |y|- b\|_2^2 \label{Gaussian}
\eeq
based on the maximum likelihood principle for the Poisson noise model. The Poisson log-likelihood function \eqref{Poisson}  is  asymptotically reduced to \eqref{Gaussian}
as shown in Appendix \ref{sec:likelihood}.

The ptychographic iterative engines, PIE \cite{PIE204}, ePIE \cite{ePIE09} and rPIE \cite{rPIE17},  are both derived from the amplitude-based cost function \eqref{Gaussian}.
The maximum likelihood scheme is also a variance stabilization scheme which uniformizes the probability distribution for every pixel regardless of the measured intensity value \cite{noise3}. 
It is well established that the amplitude-based cost function \eqref{Gaussian} outperforms the
intensity-based cost function  $\half \| |y|^2- b^2\|_2^2$ \cite{Waller15}.  See \cite{noise,noise2} for more choices of objective functions.

For non-convex iterative optimization, a good initial guess  or some regularization is usually crucial for convergence \cite{ML12,Poisson2}. 
We assume that  the initial guess for the probe satisfies  PPC. To test the linear phase ambiguity, we also consider the probe initialization
\beq
\mu_1 (\mathbf{n})=\mu^0(\mathbf{n})\,\exp{\lt[\im 2\pi\frac{ \mathbf{k}\cdot\mathbf{n}}{n}\rt]}\,\exp{[\im \phi(\mathbf{n})]},\ \ \ \mathbf{n}\in \mathcal{M}^0
\eeq
where $\phi(\bn)$ are independently and uniformly distributed on $(-\delta \pi, \delta\pi)$. We use $\delta\in [0, 1/2]$ as the control parameter for the proximity of the initial probe to the true probe, modulo a linear phase factor represented by $\bk$,
and denote this class of probe initialization by PPC($\bk,\delta$).

\commentout{The stopping  criteria in the inner loops is set as 
\begin{equation*}
\frac{\||\ |A_k A_k^\dagger u_{k}^{l}|-b\|_2-\|\ |A_k A_k^\dagger u_{k}^{l+1}|-b|
\|_2}{\|\ |A_k A_k^\dagger  u_{k}^{l}|-b\|_2} \leq 10^{-4}
\end{equation*} 
}

\commentout{Our algorithm now becomes \newline
\begin{algorithm}
\caption{DR\&AP}\label{main algorithm}
\begin{algorithmic}[1]
\State Initialize $\mathbf{\mu}_1$ (PPC), $u_{1,1}$ totally random. 
\State In epoch $k$   ($k\geq 1$), initialize $u_{k,1}= u_{k-1, l_{last}}=u_{k}$. Update $A_k=A(f_k)$ Do
    \begin{equation}
		u_{k}^{l+1} = \frac{1}{\rho+1}u_{k}^{l} + \frac{\rho-1}{\rho+1}P_k u_{k}^{l}+\frac{1}{\rho+1}b\odot \frac{2P_ku_{k}^{l}-u_{k}^{l}}{|2P_ku_{k}^{l}-u_{k}^{l}|}
		\label{iteration update image by DR}
	\end{equation}
	    until residual $\| |P_k u_{k}^{l+1} | -b\|_2$ stagnates. Update $f_{k+1}=A_k u_{k+1}$. 
\State In epoch $k$, initialize $ v_{k,1}= v_{k-1,l_{last}}= v_{k}$. Update $B_k=B(f_{k+1})$. Do 
    \begin{equation}
		 v_{k}^{l+1} = \frac{1}{\rho+1} v_{k}^{l} + \frac{\rho-1}{\rho+1}B_k\frac{B_k^\dagger v_{k}^{l}}{|B_k^\dagger v_{k}^{l}|}+\frac{1}{\rho+1}b\odot \frac{2B_k\frac{B_k^\dagger v_{k}^{l}}{|B_k^\dagger v_{k}^{l}|}-v_{k}^{l}}{|2B_k\frac{B_k^\dagger v_{k}^{l}}{|B_k^\dagger v_{k}^{l}|}-v_{k}^{l}|}
		\label{iteration update probe by DR}
	\end{equation}
		until residual stagnates. Update $\mu_{k+1}=\frac{B_k^\dagger  v_{k+1}}{|B_k^\dagger  v_{k+1}|}$. 
\State  If the residual is within tolerance, terminate; if not, go back to 2) and increase k by 1.
\end{algorithmic}
\end{algorithm}
}

\commentout{
\subsection{Gaussian log-likelihood function}
\begin{prop}
The monotonicity property holds: 
\beq
\||A_kf_k|-b\|_2 
& \ge &\|| A_{k+1}f_{k+1}|-b\|_2\\
\| |B_k \mu_{k}| -b \|_2 &\ge& 
 \||B_{k+1}\mu_{k+1}|-b\|_2.
\eeq 

\commentout{
\beq
\||A_kf_k|-b\|_2 &\geq& \| |A_kf_{k+1}|-b\|_2 \\
& = &\| |B_k \mu_{k}| -b \|_2 \\
&\geq & \||B_k\mu_{k+1}|-b\|_2\\
& = &\|| A_{k+1}f_{k+1}|-b\|_2\\
& \geq& \| |A_{x+1}f_{k+2}|-b\|_2
\eeq 
}
\label{prop:monotone}
\end{prop}
\begin{proof}
Let's review our algorithm
By the definition of the $f_{k+1}, \mu_{k+1}$.
\begin{equation*}
\begin{aligned}
&\| | A_kf_k| -b \|_2 \geq \min_x \| | A_k x | -b\|_2 = \| | A_kf_{k+1}| -b \|_2 \\
&\| | B_k\mu_k| -b \|_2 \geq \min_x \| | B_k x | -b\|_2 = \| | B_k\mu_{k+1}| -b \|_2 \\
\end{aligned}
\end{equation*}
By definition of $A_kf_{k+1} = B_k\mu_k$, we have 
\begin{equation*}
\| | A_kf_k| -b \|_2 \geq \min_x \| | A_k x | -b\|_2 = \| | A_kf_{k+1}| -b \|_2 =\| | B_k\mu_k| -b \|_2 \geq \min_x \| | B_k x | -b\|_2 = \| | B_k\mu_{k+1}| -b \|_2
\end{equation*}

Next let's estimate the first order decrease in either image update or probe update in epoch $k^{th}$. Denote the decrease of residual in image update in the $k^{th}$ epoch as 
\begin{equation*}
\Delta_k^x =  \| |A_kf_k|-b\|_2^2 -\||A_kf_{k+1}| - b \|_2^2 
\end{equation*}
We can use steepest decent and line search analysis to give an lowerbound  (first order) approximation of $\Delta_k^x$. Assume we have a perturbation $\Delta x = \epsilon u;\ \|u\|_2=1$  near $f_k$. 
\begin{equation*}
\begin{aligned}
\Delta_k^x &\geq  \| |A_kf_k| -b\|^2  - \| |A_k f_k+ \epsilon A_ku| - b \|^2\\
           & = \| |A_kf_k| -b\|^2 - \| |A_kf_k|(1+\epsilon \Re(\Omega^*_{A_kf_k}A_ku)) -b\|^2 \\
           & \approx -\epsilon(|A_kf_k|-b)^T\Re(\Omega^*_{A_kf_k}A_ku) \\
           & =-\epsilon(A_kf_k-\Omega_{A_kf_k}b)^T\Omega^*_{A_kf_k}(\Re(\Omega^*_{A_kf_k}A_k)\Re(u)-\Im(\Omega^*_{A_kf_k}A_k)\Im(u))
\end{aligned}
\end{equation*}
holds for all $u \in \mathbb{C}^{n\times n};\ \|u\|=1$ and $\epsilon$ small. We can maximize the right hand side over $u$ to get an lower bound of $\Delta_k^x$ by Cauthy inequality and taking
\begin{equation*}
u^T=-(A_kf_k-\Omega_{A_kf_k}b)^T\Omega^*_{A_kf_k}(\Re(\Omega^*_{A_kf_k}A_k)-\imath\Im(\Omega^*_{A_kf_k}A_k)) / \|.\|
\end{equation*} \newline
where $\|.\| :=\|(A_kf_k-\Omega_{A_kf_k}b)^T\Omega^*_{A_kf_k}(\Re(\Omega^*_{A_kf_k}A_k)-\imath\Im(\Omega^*_{A_kf_k}A_k))\| = \|(A_kf_k-\Omega_{A_kf_k}b)^T\Omega^*_{A_kf_k}(\bar{\Omega}^*_{A_kf_k}{A_\infty}_k)\| $
\begin{equation*}
\begin{aligned}
\Delta_k^x\geq& \epsilon(\|(A_kf_k-\Omega_{A_kf_k}b)^T\Omega^*_{A_kf_k}\Re(\Omega^*_{A_kf_k}A_k)\|^2+ \|(A_kf_k-\Omega_{A_kf_k}b)^T\Omega^*_{A_kf_k}\Im(\Omega^*_{A_kf_k}A_k)\|^2)/\|.\| \\
           = & \epsilon\|(A_kf_k-\Omega_{A_kf_k}b)^T\Omega^*_{A_kf_k}(\Re(\Omega^*_{A_kf_k}A_k)-\imath\Im(\Omega^*_{A_kf_k}A_k)) \|^2/\|.\| \\
           = & \epsilon\|(A_kf_k-\Omega_{A_kf_k}b)^T\Omega^*_{A_kf_k}(\bar{\Omega}^*_{A_kf_k}{A_\infty}_k) \|^2/\|.\|\\
           = & \epsilon\| A_k(A_kf_k-\Omega_{A_kf_k}b)\|
\end{aligned}
\end{equation*} 
By the same analysis, we can give an lower bound of decrease in probe update in the $k^{th}$ epoch,
\begin{equation*}
\Delta_k^\mu = \epsilon\| B_k(B_k\mu_k-\Omega_{B_k\mu_k}b)\|
\end{equation*}
\end{proof}
}

\subsection{DR splitting (DRS)}
We formulate the inner loops (Step 2 and 3 in Algorithm 1) of AM as 
\begin{equation}
\min\limits_{u}K(u)+\cL(u)
\label{DRS}
\end{equation}
where the additional objective function $K$ enforces the constraint that $u$ is in the range of $A_k$ or $B_k$.

We solve \eqref{DRS} by the Douglas-Rachford splitting (DRS) method 
which is a proximal point algorithm for minimizing a sum of two objective functions   \cite{ LM79,EB92}. 
This is motivated by the good performance of the classical Douglas-Rachford (CDR) algorithm \cite{DR,LM79} (aka the difference map \cite{Elser} in the optics literature), a special case of DRS with an infinitely large step size,  which has been used for ptychographic reconstruction \cite{DM08,probe09,ADM} and analyzed in \cite{FDR, ptych-unique}. 
The CDR iteration, however, exhibits only a sub-linear convergence globally in the noiseless case \cite{FDR} and, when the optimization problems become infeasible (e.g. with noisy data), tends to fluctuate  and
underperform, hence the search for a faster convergent and more robust DRS.

DRS is defined by the following iteration for $l=1,2,3\cdots$ 
\begin{equation}
\begin{aligned}
y^{l+1} &= \prox_{K/\rho}(u^l); \\
z^{l+1} &= \prox_{\cL/\rho}(2y^{l+1}-u^l)\\
u^{l+1}&= u^l+ z^{l+1}-y^{l+1}
\end{aligned}
\label{equ: DR definition iteration}
\end{equation}
where $\prox_{\cL/\rho}(u^k):= \mathop{\text{argmin}}\limits_{x} \cL(x)+\frac{\rho}{2}\| x-u^k\|^2$. 
For the initial guess $u^1_1$ in our simulations, we let $u^1_1=A_1^* f_1$ where $f_1(\bn)$ is a phase object with i.i.d. phase uniformly distributed over $[0,2\pi]$ or
a constant value 0.   Here $\gamma=1/\rho$ is often referred to as the step size of DRS. 

It can be checked that DRS \eqref{DRS} is formally equivalent to
the Alternating Direction Method of Multipliers (ADMM)
\begin{equation*}
\max_{\lambda}\min\limits_{y,z}\cL(y)+K(z)+\langle \lambda, y-z\rangle + \frac{\rho}{2}\|z-y\|^2.
\end{equation*}

Our choice of $K$ is important to our implementation of DRS.
Since the object estimate  in  the $k-$th  epoch must be of the form $A_k g$ for some $g\in \IC^{n\times n}$, 
we let $K(u)$ be the indicator function $ \chi_k(u)$ of the range of $A_k$, i.e. $\chi_k(u)=0$ if $ y$ is in the range of $A^*_k$; and $\chi_k(u)=\infty$ otherwise. 
For this choice of $K$, $ \prox_{K/\rho}(u)=P_k u$ is independent of $\rho$. This
should be contrasted with the choice of the distance function adopted in \cite{LP} for the tractability of
the convergence analysis due to the smoothness of the distance function.

For the Gaussian case \eqref{Gaussian}, we update the object estimate as 
\beq
f_{k+1}= A_k^\dagger u_{k}^{\infty}\label{49}
\eeq
where $u_k^\infty$ is the terminal output of  the following iteration
\beq
\label{G11}y_{k}^{l} &=& A_kA_k^\dagger u_{k}^{l} \\
\label{G12}z_{k}^{l} &=& \frac{1}{\rho+1}b\odot\sgn{(2y_{k}^{l}-u_{k}^{l})}+\frac{\rho}{\rho+1}(2y_{k}^{l}-u_{k}^{l}) \\
\label{G13}u_{k}^{l+1} &= &u_{k}^{l}+z_{k}^{l}-y_{k}^{l}. 
\eeq
Substituting \eqref{G11} and \eqref{G12} into \eqref{G13} and reorganizing the resulting equation, we obtain  \beq\label{G1}
u_{k}^{l+1} &=& \frac{1}{\rho+1}u_{k}^{l} + \frac{\rho-1}{\rho+1}P_k u_{k}^{l}+\frac{1}{\rho+1}b\odot \sgn\big(2P_ku_{k}^{l}-u_{k}^{l}\big)\\
&=& \half u^l_k+{\rho-1\over 2(\rho+1)} R_ku^l_k+\frac{1}{\rho+1}b\odot \sgn \big(R_ku^l_k)\nn
\eeq
where $P_k=A_kA_k^\dagger $ is the orthogonal projection onto the range of $A_k$ and
$R_k=2P_k-I$ is the corresponding reflector.  

 
\commentout{becomes the following inner loop 
\begin{equation*}
\begin{aligned}
y_{k}^{l}=&B_k\frac{B_k^\dagger v_{k}^{l}}{|B_k^\dagger v_{k}^{l}|} \\
z_{k}^{l}= & \frac{1}{\rho+1}b\odot \frac{2y_{k}^{l}- v_{k}^{l}}{|2y_{k}^{l}- v_{k}^{l}|} + \frac{\rho}{\rho+1}2y_{k}^{l}- v_{k}^{l} \\
 v_{k}^{l+1} &= v_{k}^{l}+z_{k}^{l}-y_{k}^{l}  
\end{aligned}
\end{equation*}
}
Likewise, we update the probe as 
\beq
\mu_{k+1}&=&B_k^\dagger v_{k}^{\infty}\label{54}
\eeq
where $ v_{k}^{\infty}$ is the terminal output of  the following iteration
\beq
\label{G2}
 v_{k}^{l+1}& =& \frac{1}{\rho+1} v_{k}^{l} + \frac{\rho-1}{\rho+1}Q_k v_{k}^{l}+\frac{1}{\rho+1}b\odot \sgn{\Big(2Q_k v_{k}^{l}- v_{k}^{l}\Big)}\\
 &=& \half v^l_k+{\rho-1\over 2(\rho+1)} S_ku^l_k+\frac{1}{\rho+1}b\odot \sgn \big(S_kv^l_k)\nn
\eeq 
where  $Q_k=B_kB_k^\dagger $ is the orthogonal projection onto the range of $B_k$ and $S_k$ the corresponding reflector.

For the Poisson case \eqref{Poisson},  the inner loops  take a more complicated form \beq	\label{P1}
		u_{k}^{l+1}& =&
		\half u^l_k-{1\over \rho+2} R_k u^l_k+
		      \frac{\rho}{2\rho+4}\sqrt{|R_ku_{k}^{l}|^2+\frac{8(2+\rho)}{\rho^2}b^2}\odot \sgn{\Big(R_k u_{k}^{l}\Big)}
	\eeq
	\beq
	\label{P2}
v_{k}^{l+1} &=&
\half v^l_k-{1\over \rho+2} S_k v^l_k+\frac{\rho}{2\rho+4}\sqrt{|S_k v_{k}^{l}|^2+\frac{8(2+\rho)}{\rho^2}b^2}\odot \sgn{\Big(S_k v_{k}^{l}\Big)}
\eeq
After the inner loops terminate, we update the object and probe as \eqref{49} and \eqref{54},respectively.

We shall refer to DRS with the Poisson log-likelihood function \eqref{P1}-\eqref{P2} and the Gaussian version \eqref{G1}-\eqref{G2}  by the acronyms  {\em Poisson-DRS} and {\em Gaussian-DRS}, respectively. We call the the two algorithms the Alternating Minimization with Douglas-Rachford Splitting (AMDRS). Due to the non-differentiability of both $K$ and $\cL$, the global convergence property
of the proposed DRS method is beyond the current framework of analysis \cite{LP}.

In the limiting case of $\rho\to 0$, both Gaussian-DRS and Poisson-DRS  become the classical Douglas-Rachford algorithm.

The computation involved in Gaussian-DRS and Poisson-DRS are mostly pixel-wise operations
(hence efficient) except for the pseudoinverses $A_k^\dagger = (A^*_kA_k)^{-1}A^*_k$ and  $B_k^\dagger = (B^*_kB_k)^{-1}B^*_k$. In blind ptychography,  due to the isometry of the Fourier transform, $A^*_kA_k$ and $B^*_kB^\dagger_k$ are
diagonal matrices and easy to invert \cite{blind-ptych}.

\subsection{Fixed points of Gaussian-DRS}\label{sec:fixed}

For simplicity, we shall focus the stability analysis on the fixed points of Gaussian DRS.  

Suppose that the limit $ (u,v)=\lim_{k,l \to \infty}(u_{k}^l,v_{k}^l)$ exist where $(u^l_k,v^l_k)$ are two-parameter arrays of
the Gaussian-DRS iterates. Let $\lim_{k\to \infty}A_k=A_\infty,\ \lim_{k\to \infty} B_k = B_\infty$.

 Let $P_\infty=A_\infty  A_\infty^\dagger$ and $Q_\infty=B_\infty         B_\infty^\dagger$ denote the orthogonal projection onto the range of $A_\infty$ and $B_\infty$, respectively. Let
 $P^\perp_\infty=I-P_\infty$ and $Q^\perp_\infty=I-Q_\infty$ denote the orthogonal complements of $P_\infty$ and $Q_\infty$, respectively. Denote $R_\infty=2P_\infty-I, S_\infty=2Q_\infty-I$.

Passing to the limit in \eqref{G1}-\eqref{G2} we obtain  the fixed point equation for Gaussian-DRS
\beq
u&=&  \half u+{\rho-1\over 2(\rho+1)} R_\infty u+\frac{1}{\rho+1}b\odot \sgn \big(R_\infty u)\label{sa}\\
v&=&  \half v+{\rho-1\over 2(\rho+1)} S_\infty v+\frac{1}{\rho+1}b\odot \sgn \big(S_\infty v)\label{sb}
\eeq
\commentout{
 \beq
 \label{sa}
u&=& \frac{1}{\rho+1} u+ \frac{\rho-1}{\rho+1}P_\infty  u+\frac{1}{\rho+1}b\odot \sgn\big(2P_\infty  u-u\big)\\
 v& =& \frac{1}{\rho+1} v + \frac{\rho-1}{\rho+1}Q_\infty v+\frac{1}{\rho+1}b\odot \sgn{\Big(2Q_\infty v - v\Big)} \label{sb}
\eeq
which
can be written as
 \beq
 \label{sa'}
P_\infty u+\rho P^\perp_\infty u&=&b\odot \sgn\big(2P_\infty  u-u\big)\\
Q_\infty v+\rho Q^\perp_\infty v&=& b\odot \sgn{\Big(2Q_\infty v - v\Big)}.  \label{sb'}
\eeq
}

It is straightforward to check that the true solution ($u_0:=\cF(\mu,f), v_0:=\cF(\mu,f)$) is a fixed point of  Gaussian-DRS \eqref{G1}-\eqref{G2} and
Poisson-DRS \eqref{P1}-\eqref{P2}.

We now give  theoretical motivation for our numerical choice of $\rho=1$.

 For a solution $(u,v)$ to the fixed point equation \eqref{sa}-\eqref{sb}, define the reflection images, 
 $
x: =R_\infty u, \,\, y: =S_\infty v. $
Denote the right hand side of \eqref{sa}  as $U(x)$ and consider a perturbation of $x$ by $\epsilon z$ where $\ep$ is a small positive number. 
Set ${\Omega}=\diag(\sgn(x))$, $C={\Omega}^*A_\infty$, $\eta ={\Omega}^*z$. Suppose $|x|>0$. 
From somewhat tedious but straightforward calculation, we have
\beqn
\lim_{\ep \to 0}{1\over \ep}(U(x+\epsilon v) -U(x) )&= &\Omega J_A(\eta)
\eeqn
where 
\beq
\label{38'}
J_A(\eta) &= & CC^\dagger\eta- {1\over 1+\rho}\lt[\Re\big(2CC^\dagger \eta-\eta \big) +{\imath}\big(I-\diag({b}/{|x|})\big) \Im\lt(2CC^\dagger \eta-\eta\rt)\rt].
\eeq
The differential $J_B$ of the right hand side of \eqref{sb} has the similar structure.

 We prove the following results  in Appendix \ref{sec:stable} and \ref{sec:stable2}. 
 The first result says that for $\rho=1$ all the non-solution fixed points are linearly unstable and
 the second says that for $\rho\in [0,\infty)$ the true solutions to blind ptychography  are
 stable fixed points.
 
 \begin{prop}
Let  $(u,v)$ be a solution to \eqref{sa}-\eqref{sb} with $\rho=1$ such that $|x|>0$ or $|y|>0$. Suppose 
 \beqn
\|J_A(\eta)\|_2 \le \|\eta\|_2,\quad\forall \eta\in \IC^N&\mbox{or} & \|J_B(\xi)\|_2\le \|\xi\|_2,\quad\forall \xi \in \IC^N
   \eeqn
  where $N $ is the dimension of data. 
Then
$ f_\infty :=A_\infty^\dagger u$ and $\mu_\infty:=B_\infty^\dagger v$
are a solution to blind ptychography, i.e.  $|\cF(\mu_\infty,f_\infty)| = b$. 
Moreover, we have 
\beq
\label{key2}
x=A_\infty f_\infty=b\odot\sgn(x)=B_\infty \mu_\infty =y.
\eeq

\label{thm:stable}
\end{prop}

\begin{prop}\label{thm:stable2}
Let $(u,v)$ be a solution to \eqref{sa}-\eqref{sb} for any $\rho\in [0,\infty)$ such that $|\cF(\mu_\infty,f_\infty)| = b$ where $\mu_\infty$ and $f_\infty$
are defined as in Proposition \ref{thm:stable}. Then $\|J_A(\eta)\|_2 \le \|\eta\|_2, \|J_B(\xi)\|_2\le \|\xi\|_2$ for all $\eta,\xi \in \IC^N$ and the equality holds 
in the direction $\pm \imath b/\|b\|$ (and possibly elsewhere on the unit sphere). 
\end{prop}

Moreover, under the uniqueness conditions of \cite{blind-ptych}, 
Proposition \ref{thm:stable} and \ref{thm:stable2} then imply that the true solution and its inherent ambiguities (modulo a scaling factor
and affine  phase ramp) are the only stable fixed points of Gaussian-DRS with the unit step size. 
Therefore Gaussian-DRS with the unit step size should not stagnate near a fixed point
that is not a solution to blind ptychography.

\section{Experiments}\label{sec:num}

As motivated by Proposition \ref{thm:stable}, we fix the DRS parameter $\rho=1$ in all simulations. As a result, \eqref{G1}-\eqref{G2} become
\beqn
u_{k}^{l+1} &= &\frac{1}{2}u_{k}^{l} +\frac{1}{2}b\odot \sgn\big(R_ku_{k}^{l}\big)\\
v_{k}^{l+1}& =& \frac{1}{2} v_{k}^{l}+\frac{1}{2}b\odot \sgn{\Big(S_k v_{k}^{l}\Big)}. 
\eeqn
and \eqref{P1}-\eqref{P2} become
 \beqn u_{k}^{l+1}& =& \frac{1}{2}u_{k}^{l} -\frac{1}{3}R_k u_{k}^{l}+\frac{1}{6}\sqrt{|R_k u^l_k|^2+24b^2}\odot \sgn{\Big(R_k u_{k}^{l}\Big)}\nn \\
 v_{k}^{l+1} &=& \frac{1}{2} v_{k}^{l} -\frac{1}{3}S_k v_{k}^{l}+\frac{1}{6}\sqrt{|S_k v_{k}^{l}|^2+24b^2}\odot \sgn{\Big(S_k v_{k}^{l}\Big)}.\nn
\eeqn

\subsection{Test objects}
\begin{figure}[t]
\centering
\subfigure[]{\includegraphics[width=6.5cm]{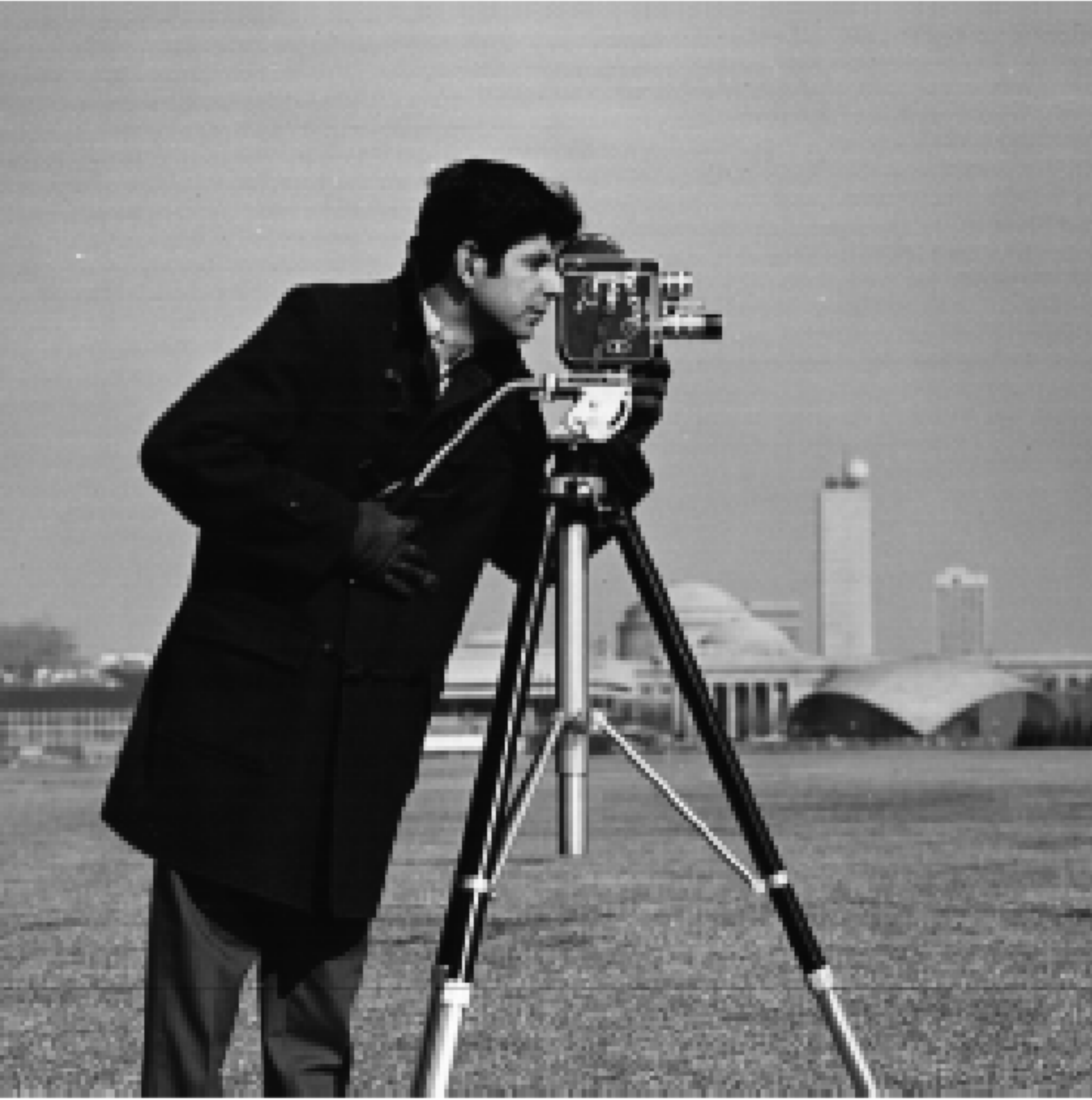}}\hspace{1cm}
\subfigure[]{\includegraphics[width=6.5cm]{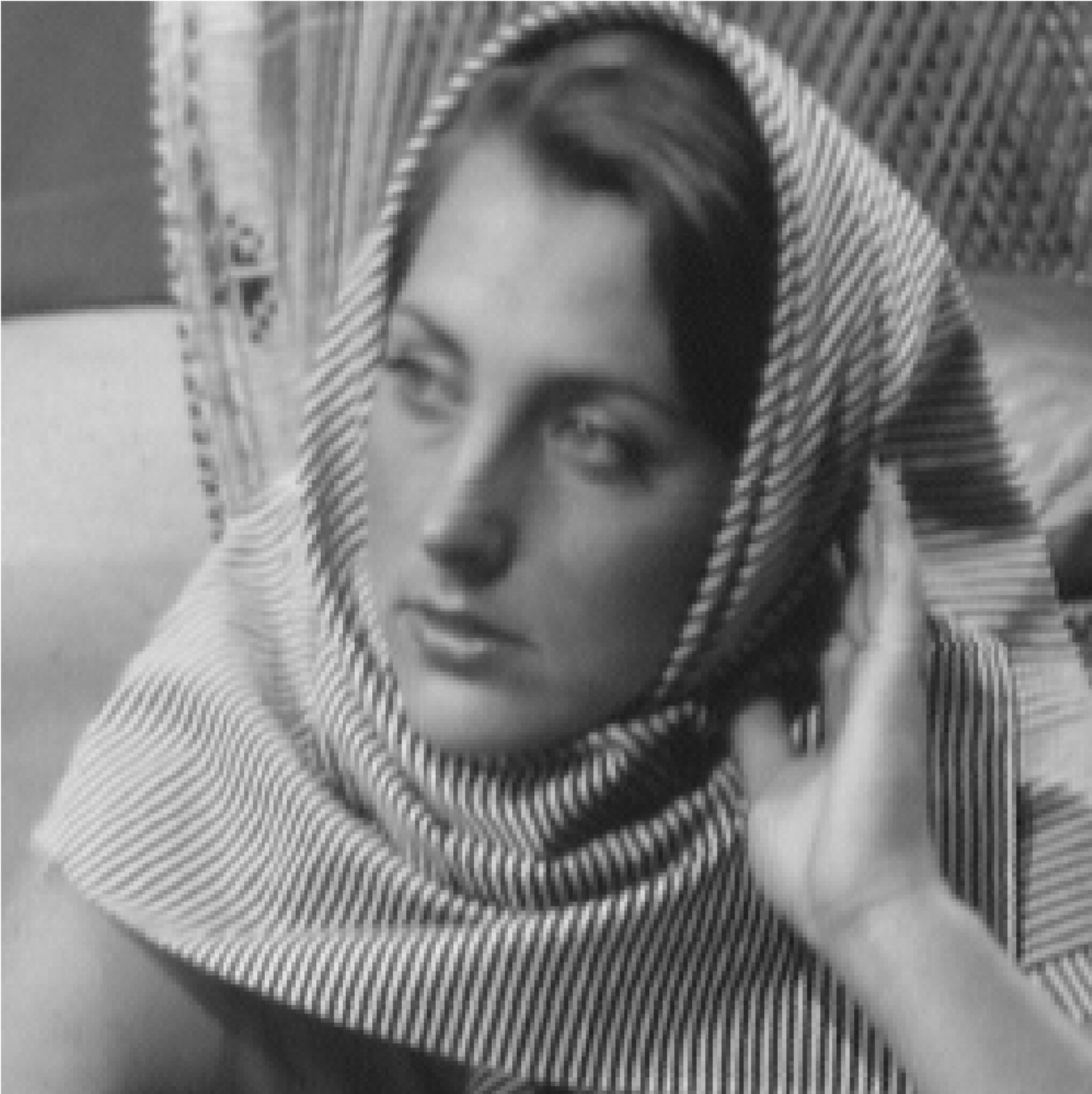}}
  \caption{(a) The real part and (b) the imaginary part of the first test image.}
  \label{fig:image}
\end{figure}%

\begin{figure}[h]
\centering
 \subfigure[RPP magnitudes]{\includegraphics[width=6.5cm]{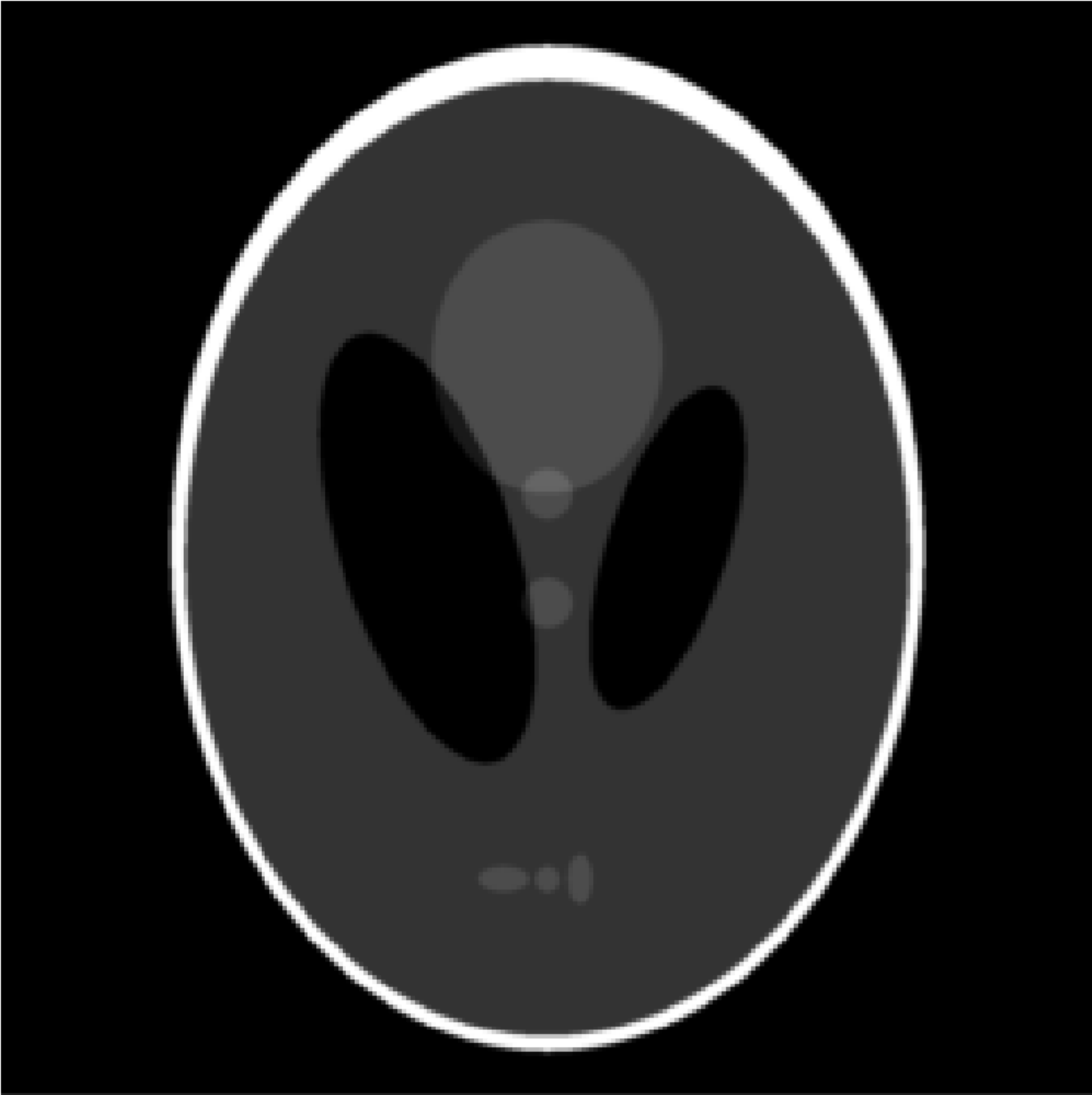}}\hspace{0.7cm}
\subfigure[RPP phases ]{\includegraphics[width=7cm,height=6.8cm]{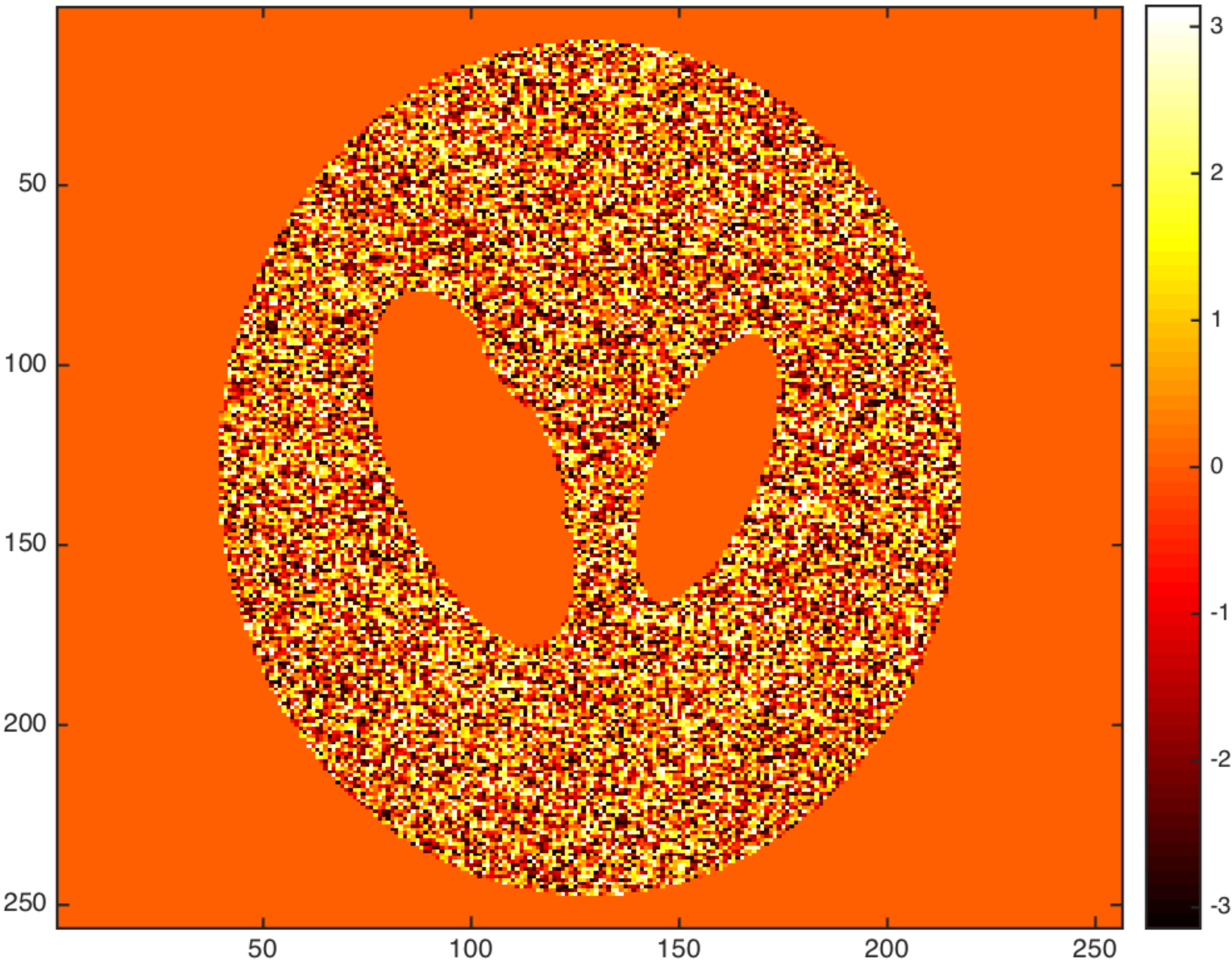}}  \\
 \caption{(a) Magnitudes and (b) phases of  RPP. }
 \label{fig:RPP}
 \end{figure}

\commentout{
\begin{figure}
\centering
\subfigure[corn grains]{\includegraphics[width=7cm]{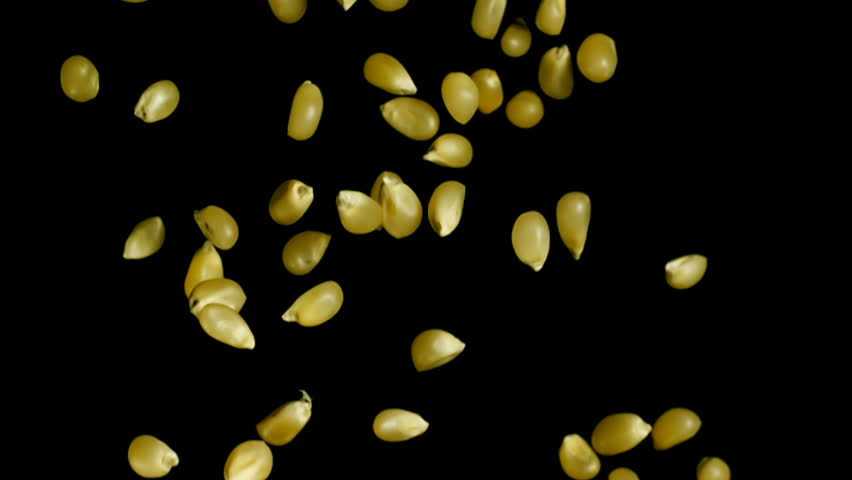}}
\caption{}
\label{fig:corn}
\label{fig11}
\end{figure}
}
Our first test image  is 256-by-256 Cameraman+ $\im$ Barbara (CiB). The resulting test object has the phase range $\pi/2$. 
The second test object is randomly-phased phantom (RPP) defined by $
f_*=P\odot  e^{\im\phi}
$
where $P$ is the standard phantom (Fig. \ref{fig:RPP}(a)) and  $\{\phi(\bn)\}$ are i.i.d. uniform random variables over $[0,2\pi]$. 
RPP has the maximal phase range because of its noise-like phase profile.  In addition to the huge phase range, RPP has loosely supported parts with respect to the measurement schemes (see below) due to its thick dark margins around the oval. 

\commentout{
The most disconnected object that we consider here is the corn-grain image (Fig. \ref{fig11}(a) ), which
easily violates the strong connectedness condition \eqref{s-conn}. Again we compare the results of the corn grains
and the salted version to see how the strong connectedness condition \eqref{s-conn} affects the performance and
explore the possibility of 
adding salt as an effective tool for ptychographic reconstruction of disconnected objects. 
}

\subsection{Measurement schemes}\label{sec:scan}

\begin{figure}
\centering
\subfigure[i.i.d. probe]{\includegraphics[width=7cm]{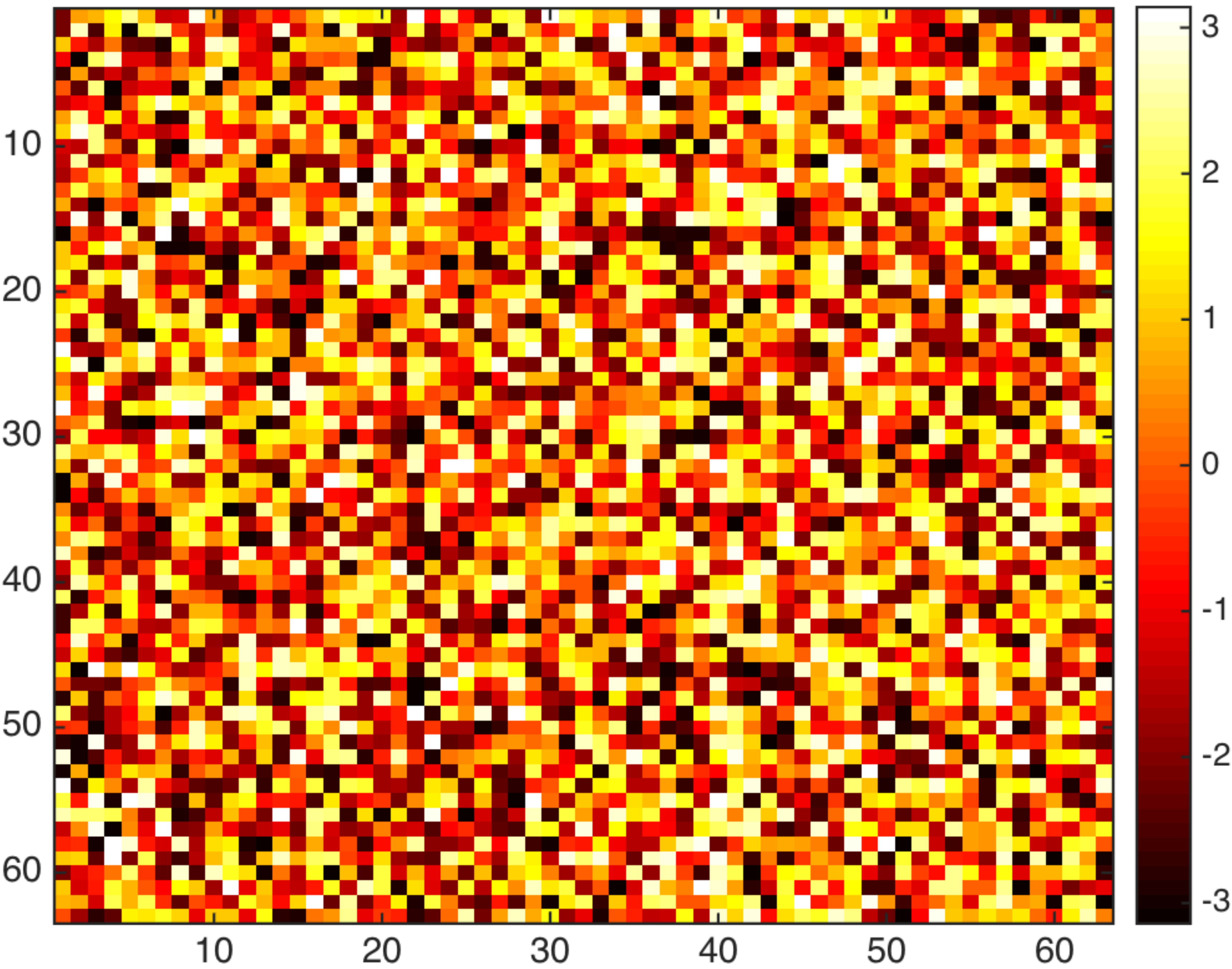}}\quad
\subfigure[Correlated probe $c=0.4$]{\includegraphics[width=7cm]{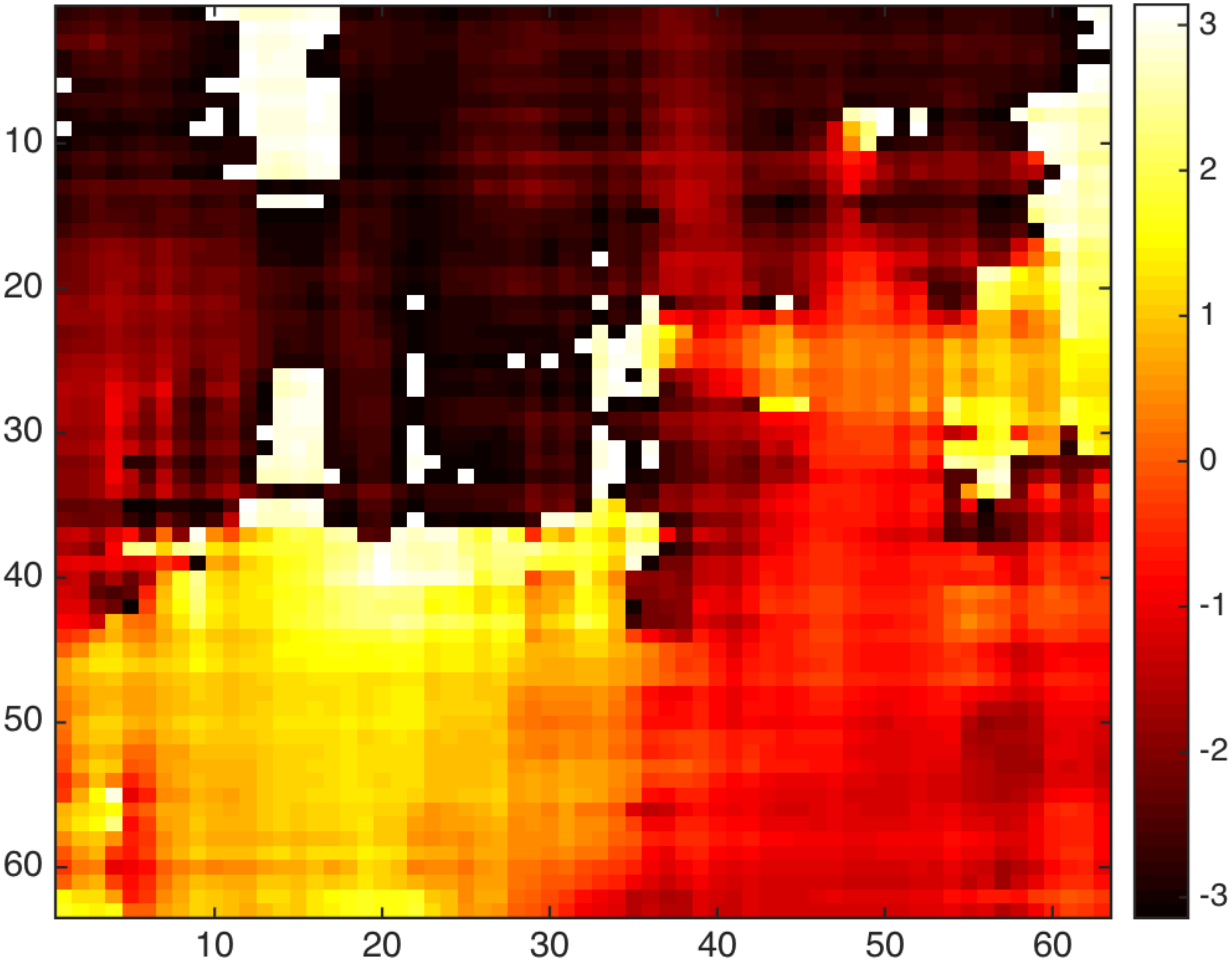}}\\
\subfigure[Correlated probe $c=0.7$]{\includegraphics[width=7cm]{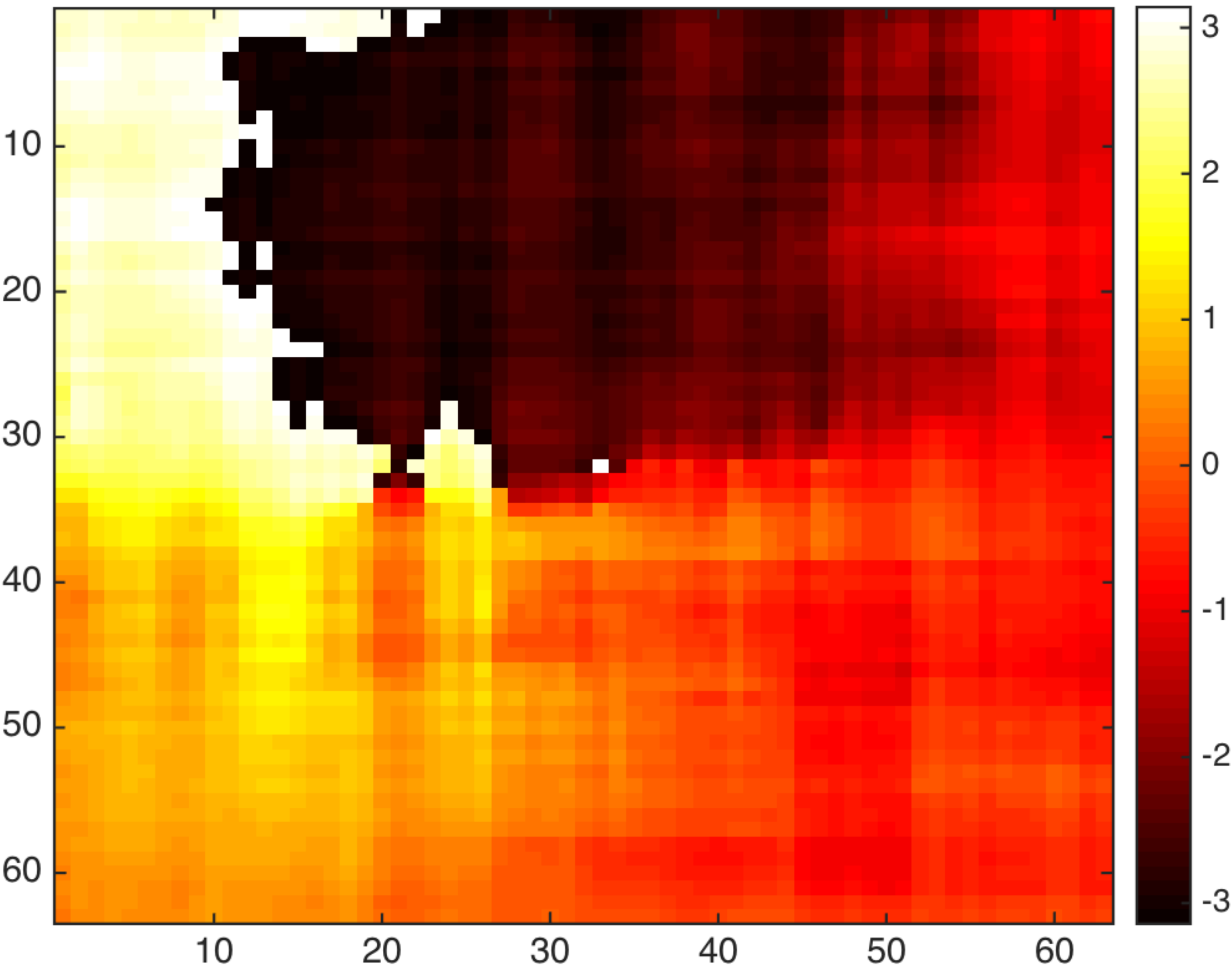}}\quad
\subfigure[Correlated probe $c=1$]{\includegraphics[width=7cm]{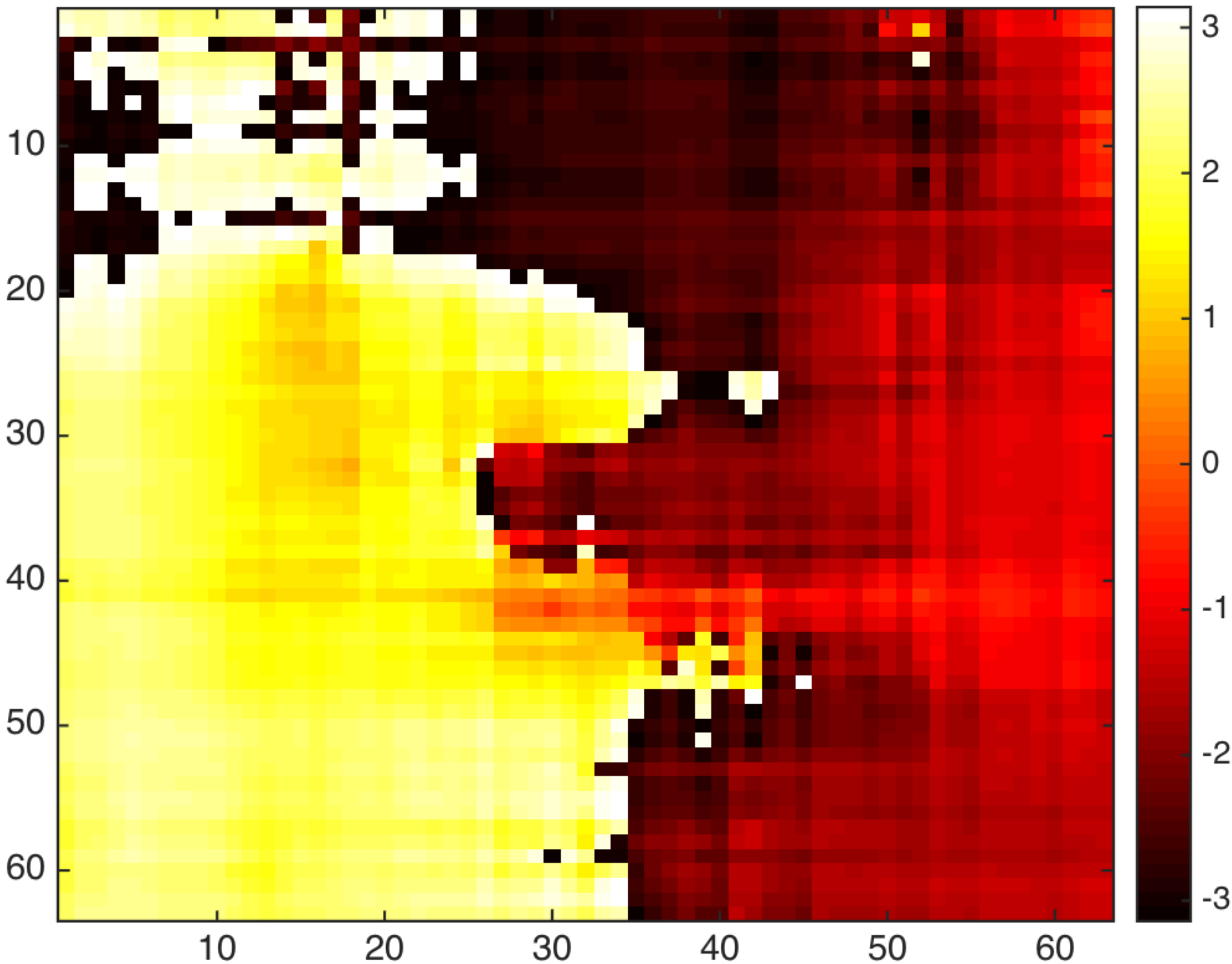}}
\caption{The phase profile of (a) the i.i.d. probe and (b)(c)(d) the correlated probes of various correlation lengths.}
\label{fig:probes}
\end{figure}
We do not explore the issue of varying 
the probe size in the present work, which has been done with the classical Douglas-Rachford algorithm  in \cite{ptych-unique}.
We fix the probe size to $60\times 60$. In addition to the i.i.d. probe, we consider also correlated probe
produced by convolving the i.i.d. probe with characteristic function of the set $\{(k_1,k_2 )\in \mathbb{Z}^2 : \max\{|k_1|,|k_2|\} \leq c\cdot m;\ c\in (0,1]  \}$ where  the constant $c$ is a measure
of the correlation length in the unit of $m=60$ (Fig. \ref{fig:probes}). 

We let  $\delta_{k}^1$ and $\delta^2_l$ in the rank-one scheme \eqref{rank1} and 
 $\delta_{kl}^1$ and $\delta^2_{kl}$ in the full-rank scheme \eqref{rank2} to be {i.i.d. uniform random variables  over $\lb -4,4\rb$}. 
In other words, the adjacent probes overlap by an average of $\tau/m=50\%$.

\subsection{Error metrics}

 We use relative error (RE) and relative residual (RR) as the merit metrics  for the recovered image $f_k$ and probe $\mu_k$ at the $k^{th}$ epoch:
\beq\label{RE}
\mbox{RE}(k)&= &\min_{\alpha\in \mathbb{C}, \mathbf{k} \in \mathbb{R}^2}\frac{\|f(\mathbf{k}) - \alpha e^{-\imath {2\pi}\mathbf{k}\cdot \mathbf{r}/n} f_k(\mathbf{k})\|_2}{\|f\|_2}\\
\mbox{\rm RR}(k)&= & \frac{\|b- |A_kf_{k}|\|_2}{\|b\|_2}.
\eeq
Note that in \eqref{RE} the linear phase ambiguity is discounted along with a scaling factor.

\subsection{Different combinations}
First we compare performance of AMDRS with different combinations of objective functions, scanning schemes and
random probes in the case of  noiseless measurements with  the periodic boundary condition. 
We use the stopping  criteria for the inner loops: 
\begin{equation*}
\frac{\| |A_k A_k^\dagger u_{k}^{l}|-b\|_2-\||A_k A_k^\dagger u_{k}^{l+1}|-b\|_2}{\||A_k A_k^\dagger u_{k}^{l}|-b\|_2} \leq 10^{-4}
\end{equation*} 
with the maximum number of iterations capped at 60.

\begin{figure}
\centering 
\subfigure[]{\includegraphics[width=8cm]{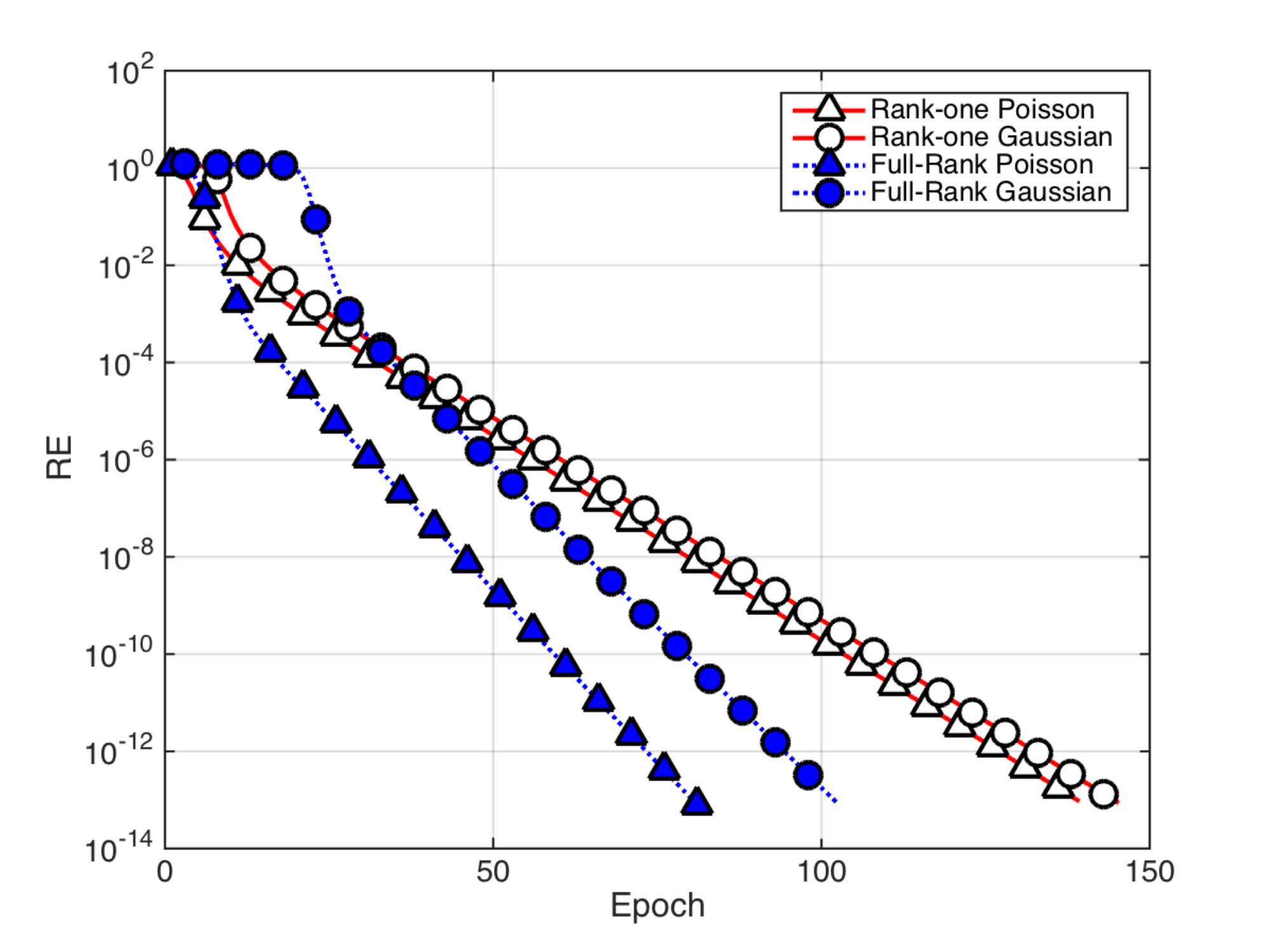}}
\centering
\subfigure[]{\includegraphics[width=8cm]{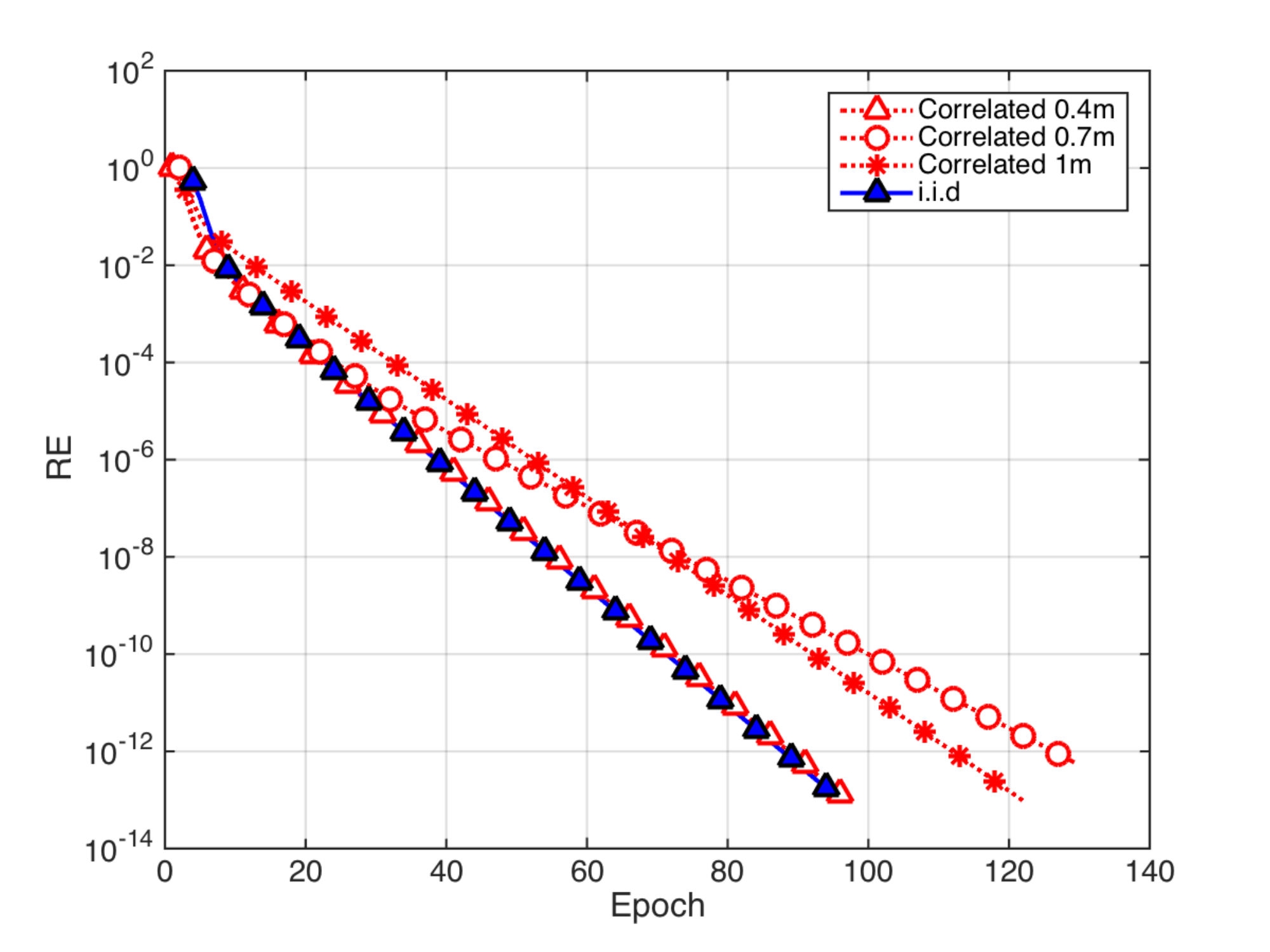}} \\
  \caption{ Geometric convergence to CiB at various rates for (a) Four combinations of objective functions and scanning schemes with i.i.d. probe (rank-one Poisson, $r_o=0.8236$; rank-one Gaussian, $r_o=0.8258$; full-rank Poisson, $r_o=0.7205$; full-rank Gaussian, $r_o=0.7373$) and (b) Poisson-DRS with four probes of different correlation lengths ($r_o=0.7583$ for $c=0.4$; $r_o=0.8394$ for  $c=0.7$; $r_o=0.7932$ for $c=1$; $r_o=0.7562$ for iid probe)
  }
  \label{fig1}
  \label{fig2}
\end{figure}%

Figure \ref{fig1} shows geometric decay of RE \eqref{RE}  at various rates $r_o$ for the test object CiB. 
In particular, Fig. \ref{fig1}(a) shows that the full-rank scheme outperforms the rank-one scheme
and that Poisson-DRS outperforms (slightly) Gaussian-DRS while
Figure \ref{fig2}(b) shows that the i.i.d. probe yields the smallest rate of convergence ($r_o=0.7562$) closely followed by
the rate ($r_o=0.7583$)  for $c=0.4$. 
\commentout{
In contrast, the algorithm does not
converge for 
the Fresnel probe $\mu^0(k_1,k_2) :=\exp\{\im \pi \gamma (k_1^2+k_2^2)/m\}$. 
}

\commentout{
\subsection{Linear phase ambiguity}

To see if the affine phase ambiguity is a fixed point of the algorithm, we test the initial guess $\mathbf{\mu_1} := \mathbf{\mu^0}\sgn\big(\sum_{0}^5 e^{\imath\theta_i} { e^{\imath 2\pi\mathbf{k}_i\cdot \mathbf{m}/n}}\big)$, where $ \theta_i \sim {\rm i.i.d.}(0,2\pi)$ and 
$ \mathbf{k}_1=(-1,1); \mathbf{k}_2=(1,0); \mathbf{k}_3=(0,1); \mathbf{k}_4= (1,1); \mathbf{k}_5= (1,2)$. 
}

\commentout{
\begin{figure}[h]

\centering
  \subfigure[$\mathbf{k}_f=(1,1), r_r=0.7418$, $r_o=0.74,r_m=0.7420$]{\includegraphics[width=8cm]{figs/LPA2.png}}\quad
  \subfigure[$\mathbf{k}_f=(3,0), r_r=0.7426$$r_o=0.7421, r_m=0.7429$]{\includegraphics[width=8cm]{figs/LPA3.png}}
  \caption{ 
  }
\end{figure}
}

 \commentout{
 \begin{figure}
 \centering
 \subfigure[$\mathbf{k}_{o}=(0,1.231),  r_o=0.9345, r_m=0.9376$ ]{\includegraphics[width=8cm]{figs/phantom-converge.png}}\quad 
  \subfigure[$\mathbf{k}_f=(1,1),  r_o=0.9115, r_m=0.9153$]{\includegraphics[width=8cm]{figs/phantom-salt-converge.png}}
 \caption{Convergence to (a) RPP and (b) salted RPP.}
 \commentout{
 { (a),  the convergence rate is 0.922 for RR, 0.914 for RE (object) and 0.93 for RE (probe).
The optimal wavenumber for the probe estimate is $(-1,0.283)$. 
For salted RPP, the convergence rate is 0.898 for RR, 0.893 for RE (object) and 0.897 for RE (probe). The optimal wavenumber for the probe estimate is $(-1,-3)$.}
}
\label{fig:RPP2}
\end{figure}
}
 \subsection{Poisson noise}
 
\begin{figure}[t]
\centering
  \includegraphics[width=10cm]{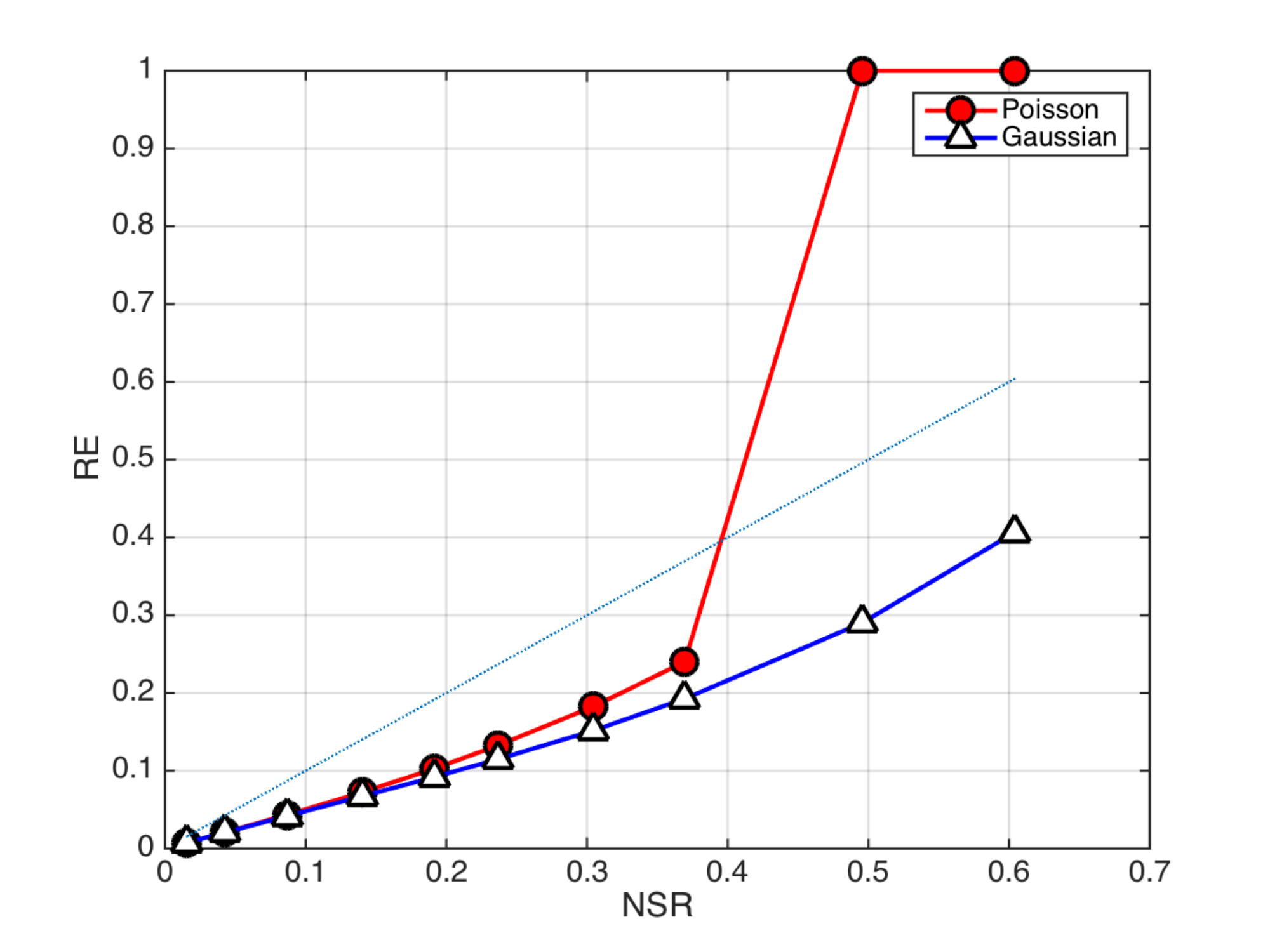}
  \caption{RE versus NSR}
  \label{fig: NSR verse RR}
\end{figure}%

For noisy measurement, the level of noise is measured in terms of the noise-to-signal ratio (NSR).
\begin{equation*}
{\rm NSR}= \frac{\|b-|A_0f_0|\|_2}{\|A_0f_0\|_2}
\end{equation*}
Figure \ref{fig: NSR verse RR} shows RE  \eqref{RE} for CiB versus  NSR for Poisson-DRS and Gaussian-DRS with the periodic boundary condition, i.i.d. probe and the full-rank scheme. The maximum number of epoch in AMDRS is limited to $100$. The RR stabilizes  usually  after 30 epochs. The (blue) reference straight line has slope 1. We see that the Gaussian-DRS outperforms the  Poisson-DRS, especially when  the Poisson RE becomes unstable for NSR $\ge 35\%$.  As noted in \cite{rPIE17,adaptive,AP-phasing} fast convergence (with the Poisson log-likelihood function) may introduce noisy artifacts and reduce reconstruction quality. 

For the rest of the experiments, we use noiseless data, Poisson-DRS and the full-rank scheme.  

\subsection{Boundary conditions}\label{sec:BC}
\begin{figure}
\centering
\subfigure[CiB reconstruction with PPC(0, 0, 0.5)]{\includegraphics[width=8cm]{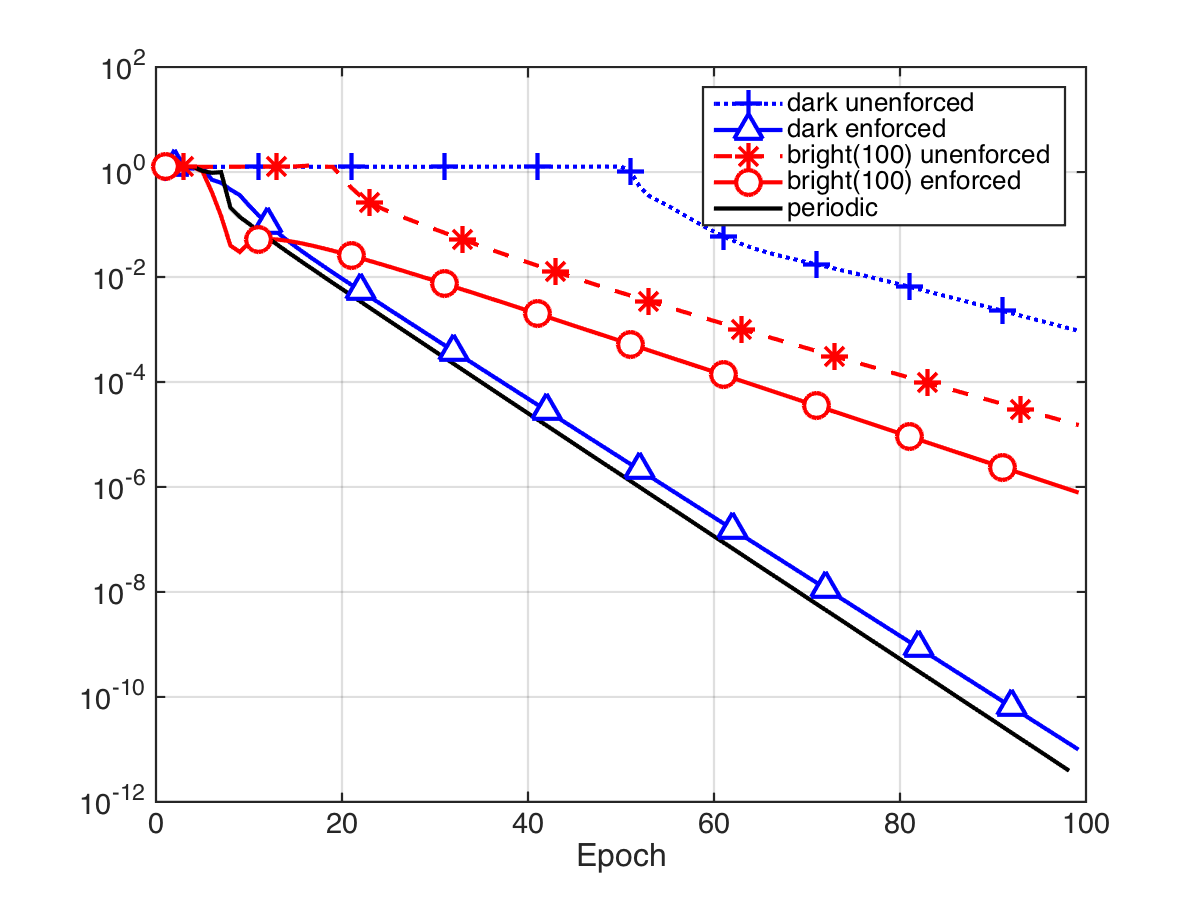}}
\subfigure[RPP reconstruction with PPC(0, 0, 0.4)]{\includegraphics[width=8cm]{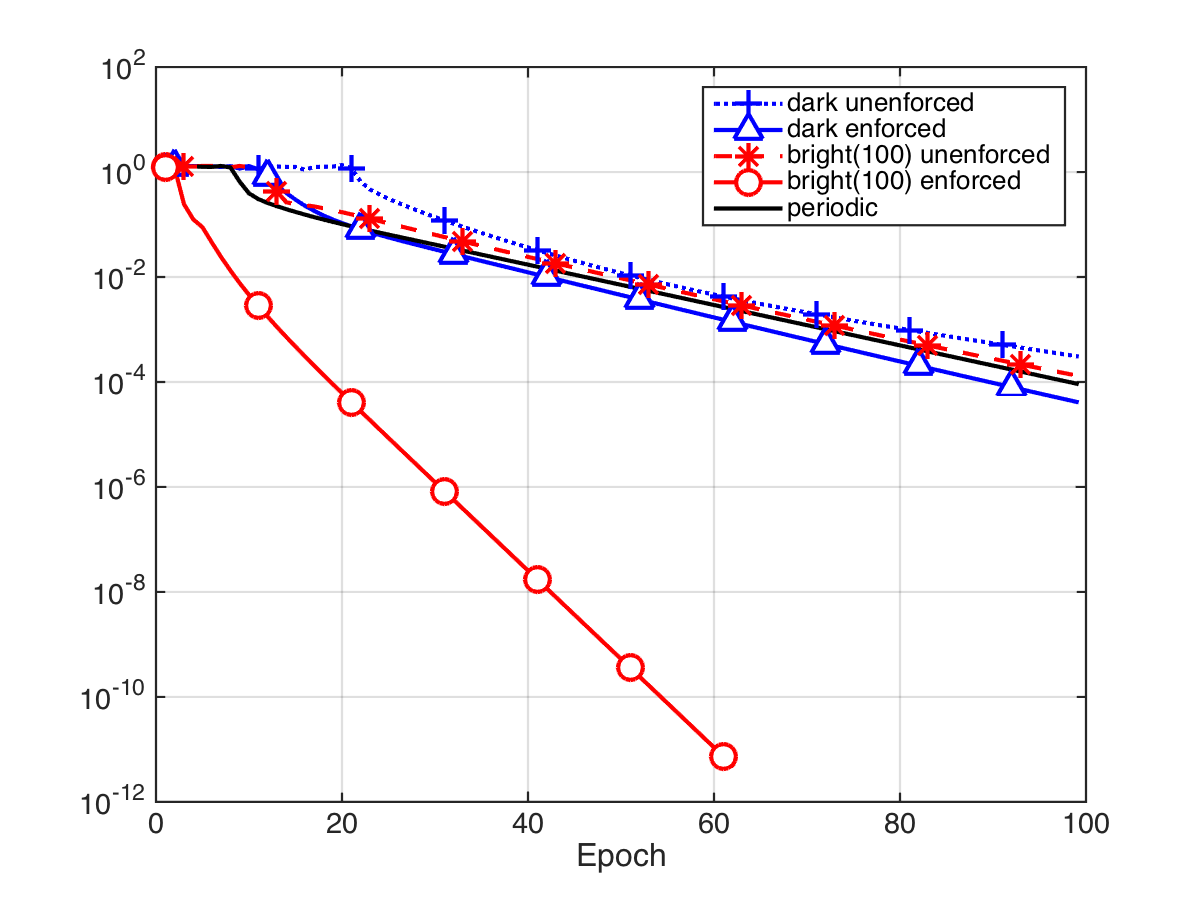}}
\caption{RE under  various boundary conditions}
\label{fig:BC}
\end{figure}

\begin{figure}
\subfigure[Reconstructed moduli with dark BC]{\includegraphics[width=6.7cm,height=6.7cm]{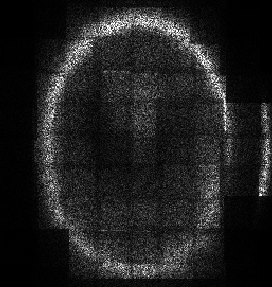}}\quad\quad\quad
\subfigure[Reconstructed phase error with periodic BC]{\includegraphics[width=7.3cm,height=7cm]{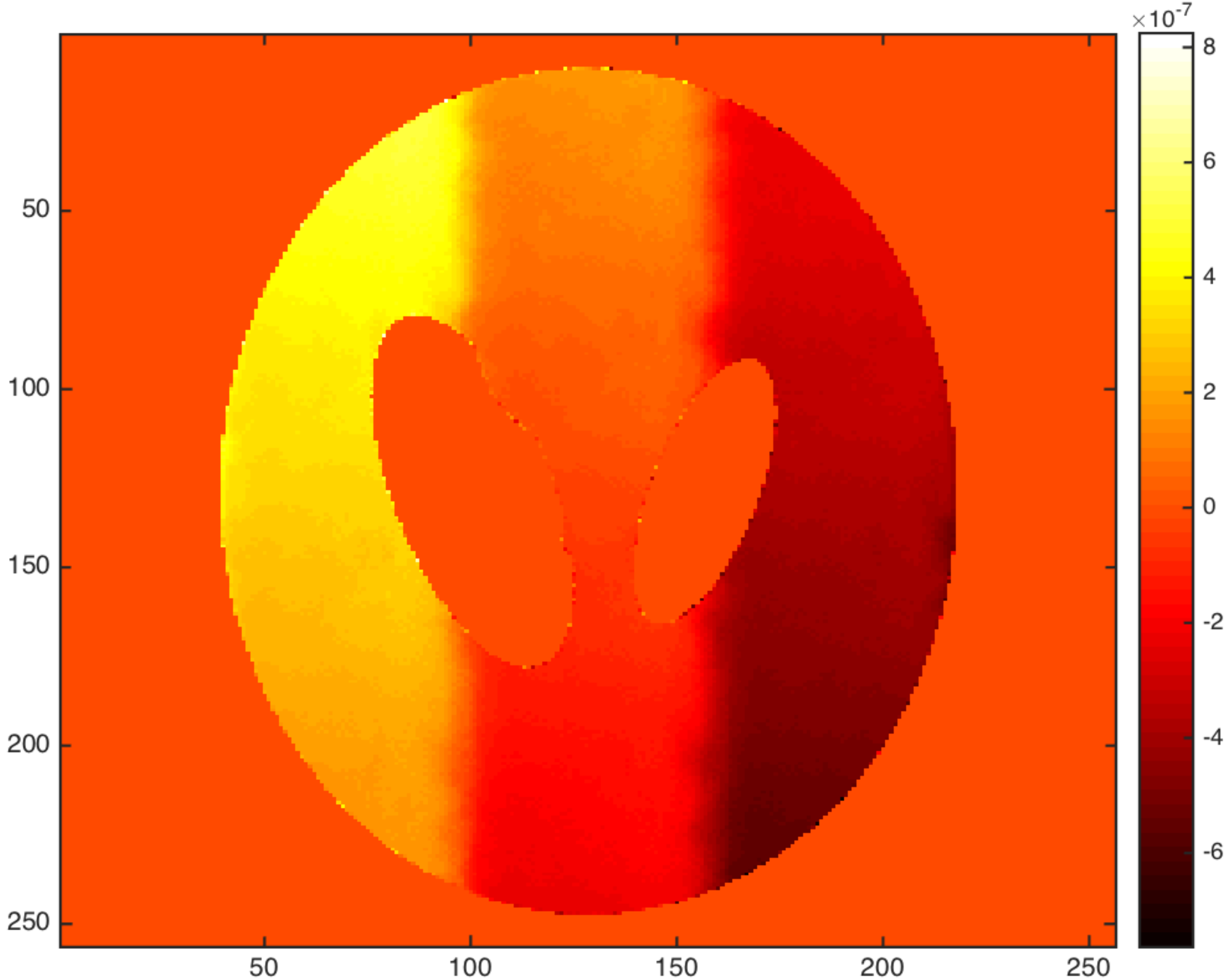}}
\caption{RPP reconstruction with PPC(0, 0, 0.5) and (a) the dark-field BC and
(b) the periodic BC. }
\label{fig:dark}
\end{figure}

\commentout{10/29/2018
\begin{figure}
\centering
 \subfigure[Moduli of salted RPP]{\includegraphics[width=.4\linewidth]{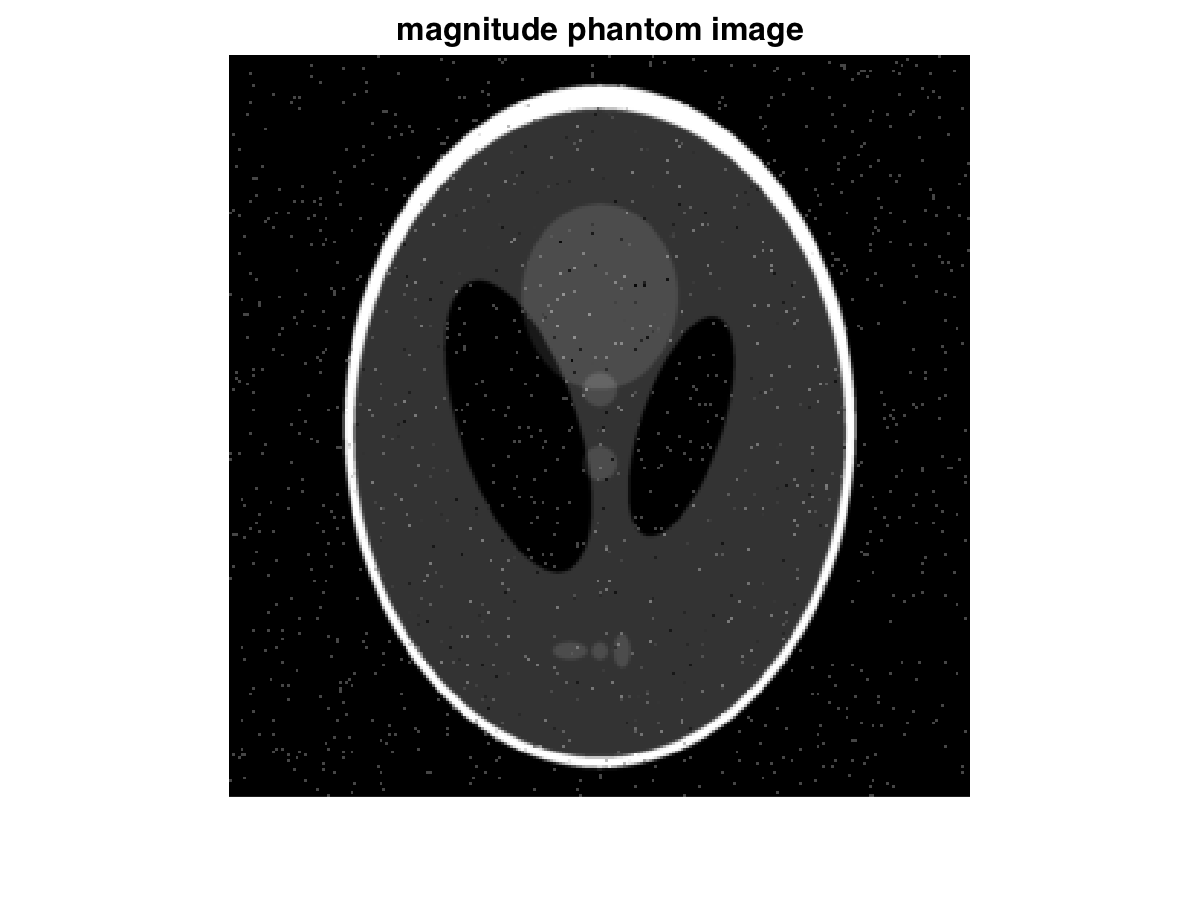}}
 \subfigure[Phases of salted RPP ]{\includegraphics[width=.4\linewidth]{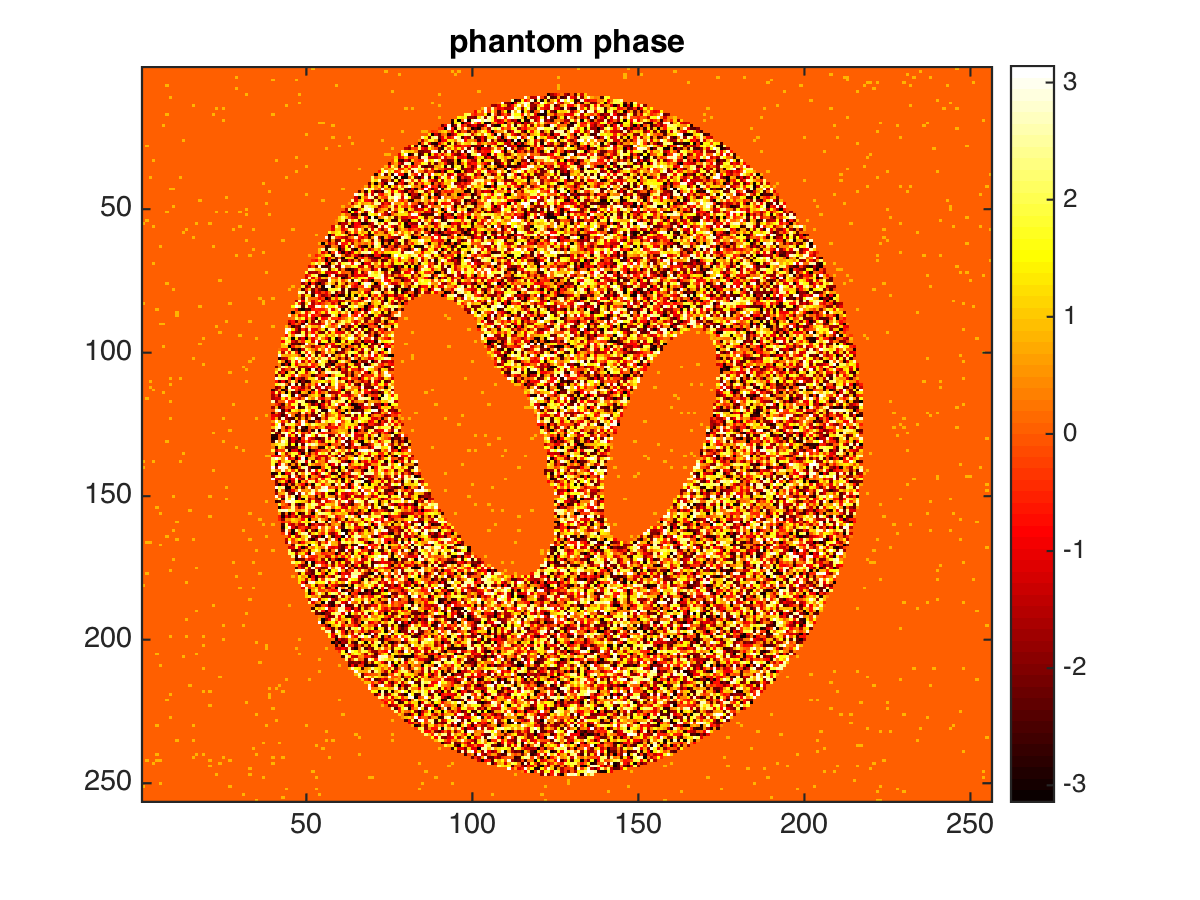}}\\
 \subfigure[Reconstructed moduli]{\includegraphics[width=.4\linewidth]{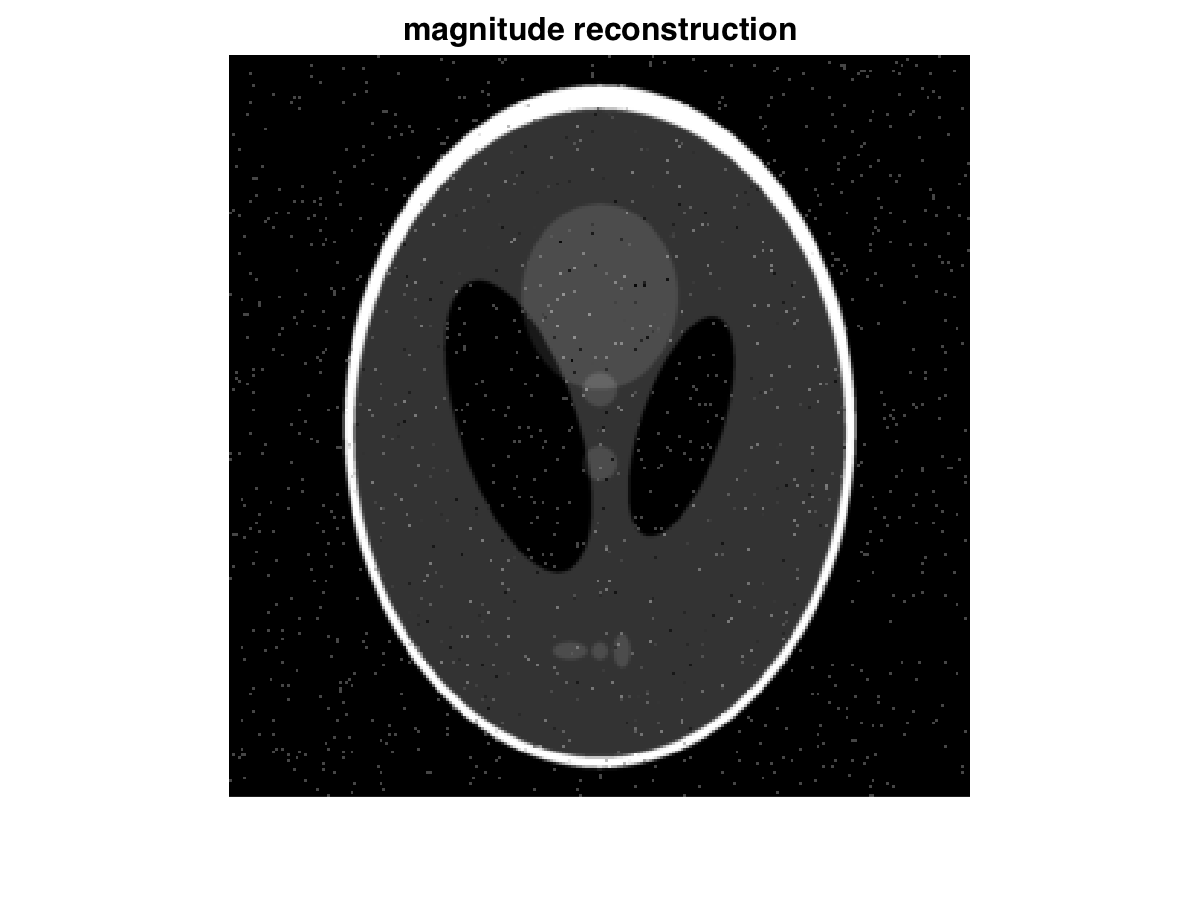}}
 \subfigure[Object phase errors]{\includegraphics[width=.4\linewidth]{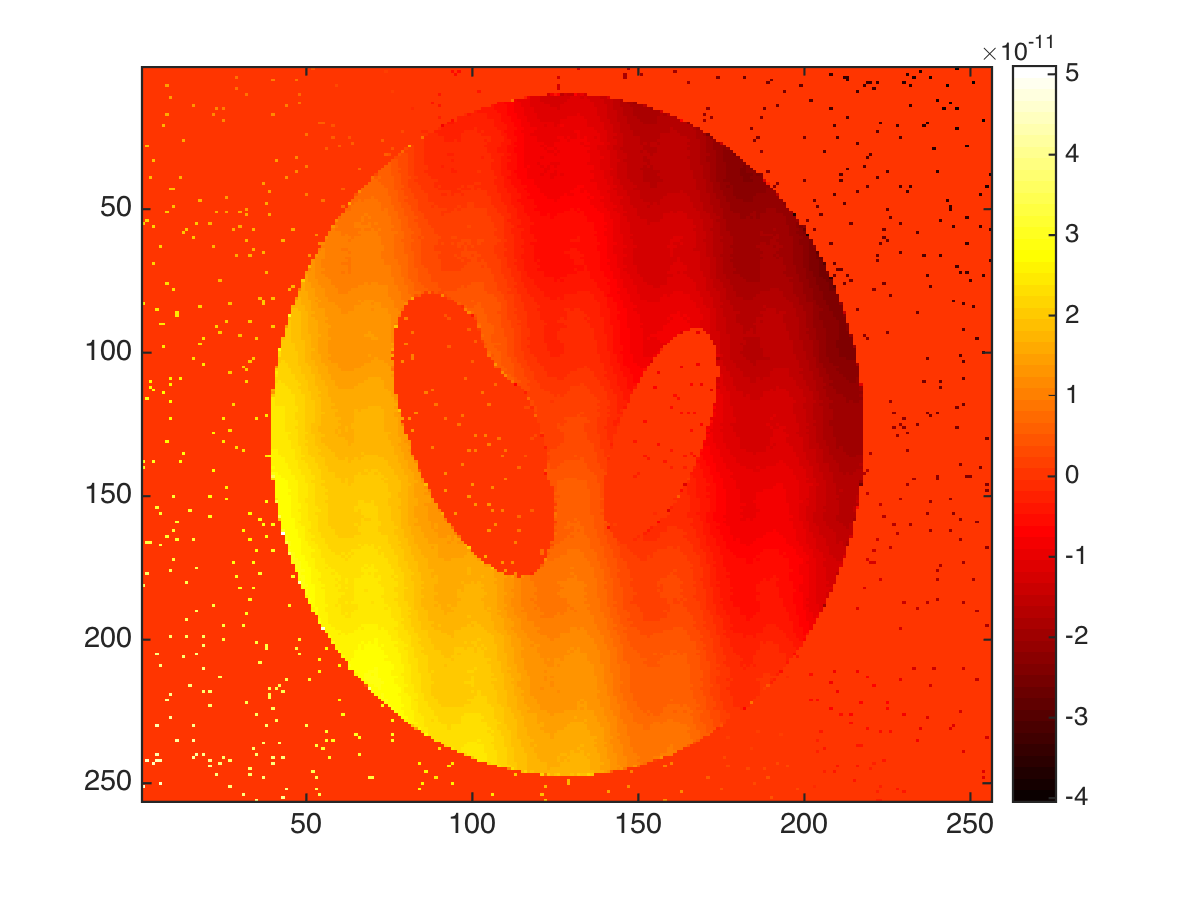}}\\
 \subfigure[Probe phase errors]{\includegraphics[width=.4\linewidth]{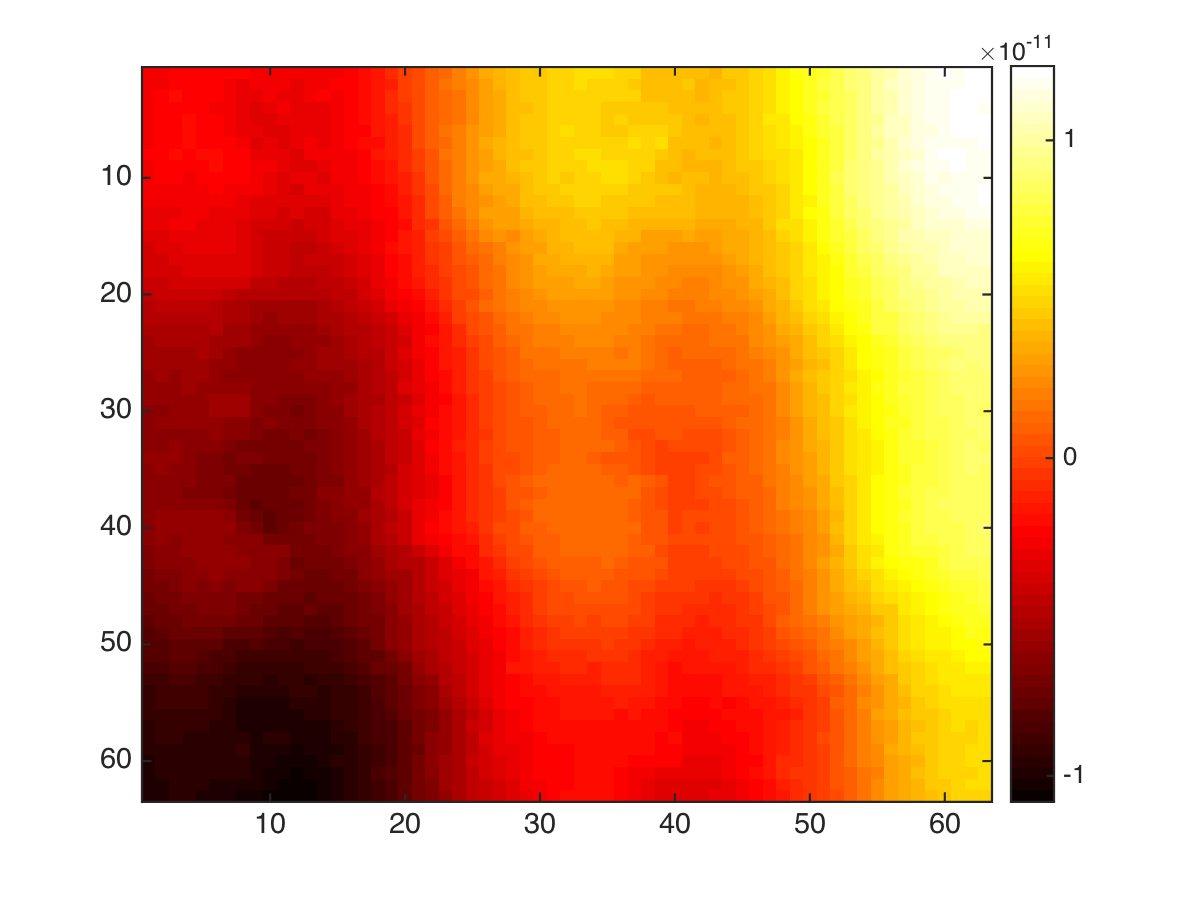}}
 \subfigure[6]{\includegraphics[width=.4\linewidth]{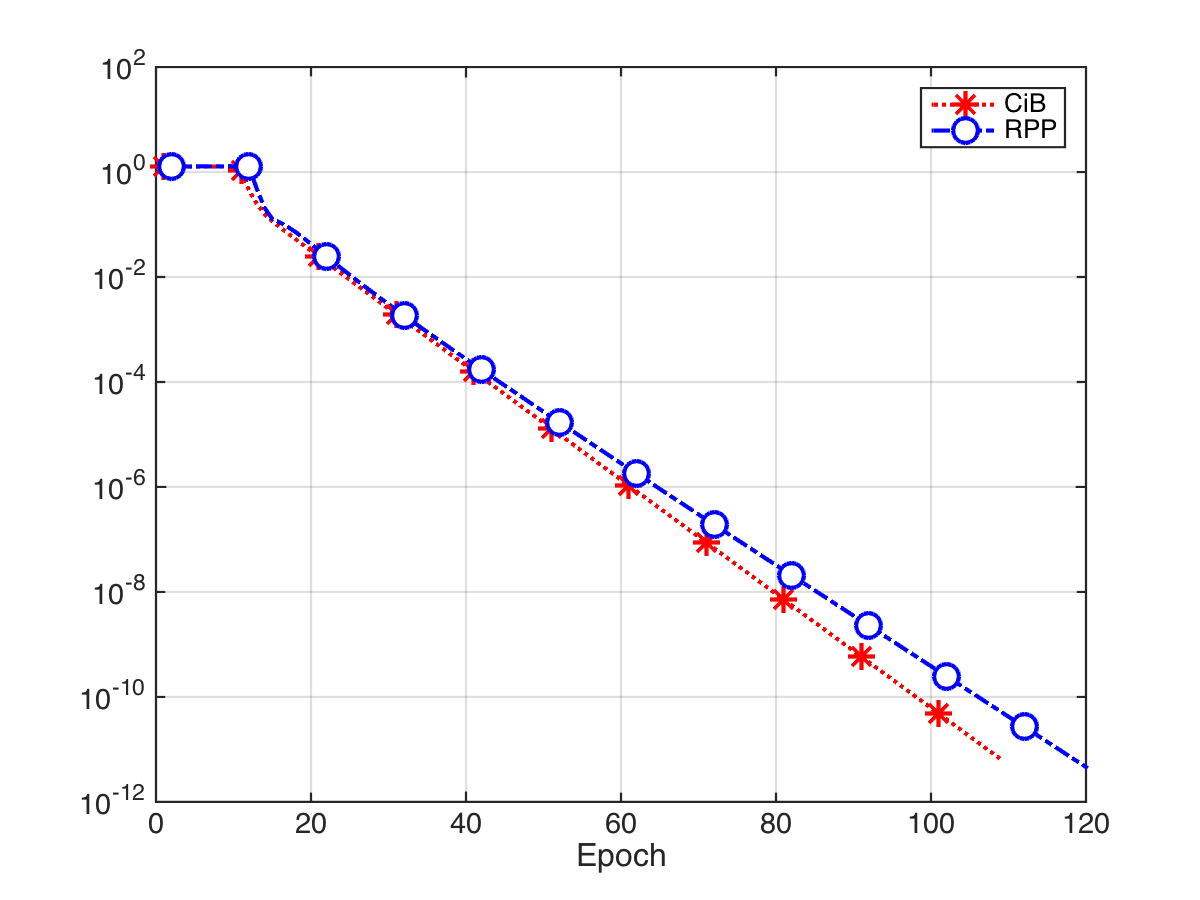}}
\caption{PPC(-1,1,0.5)  RPP and CiB: Blue circle: Rate = 0.8015, $\mathbf{k}_{im}=(-1,1.0517)$; Red star: Rate=0.7787, $\mathbf{k}_{im}=(-1,1)$}
\label{fig: Linear phase shift, phantom test, salt noise}
\end{figure}
}

When the probe steps outside of the boundary of the object domain, the 
area $\cM\setminus \IZ_n^2$ 
needs special treatment in the reconstruction process. 

The periodic boundary condition conveniently  treats all diffraction patterns and object pixels
in the same way by assuming that $\IZ_n^2$ is a (discrete) torus. The periodic boundary condition generally forces 
the slope $\br$ in the linear phase ambiguity to be integers.
The dark-field and  bright-field  boundary conditions assume zero and nonzero values, respectively, in $\cM\setminus \IZ_n^2$. 

Depending on our knowledge of the boundary values, we may or may not enforce the boundary condition in each case.  
When the bright-field boundary condition is enforced,  the linear phase ambiguity disappears from the object estimate. On the contrary, enforcing the dark-field boundary condition can not remove
the linear phase ambiguity. In both cases, however, enforcing either boundary condition speeds
up the convergence as shown in Figure \ref{fig:BC} which is produced by keeping the maximum iterations of the inner loops to 30. 

Moreover, the dark-field boundary condition can pose a challenge for reconstruction if the object domain has many dark pixels as in the case of RPP. Fig. \ref{fig:dark}(a) shows an example of
failed reconstruction with PPC(0,0,0.5) of RPP which  has a piece of the shell shifted sideway. 
Two ways to fix the problem: One is to improve the initialization condition and the other is to
use a different boundary conditions (periodic or bright-field BC). Fig. \ref{fig:dark}(b) shows the
reconstructed phase error with the periodic BC (note the scale of the color bar).  Fig. \ref{fig:BC}(b) shows the relative error with probe initialization PPC(0, 0, 0.4) and various boundary conditions. 

 \commentout{
However, its rate of convergence and
 accuracy of reconstruction depends on the contrast between the boundary value and the object values. 
 The larger the contrast, the greater the basin of attraction  and the faster the convergence for the iteration. 

On the contrary, the dark-field boundary condition assumes the zero value for the pixels outside of the object domain
can neither remove the linear phase ambiguity nor forcing the linear phase slope $\br$ to be integers. As a result, the limit of vanishing bright-field boundary condition 
is a singular one.  }

\commentout{

\subsection{Dark pixels}
\commentout{
 \begin{figure}
 \centering
{\includegraphics[width=9.5cm]{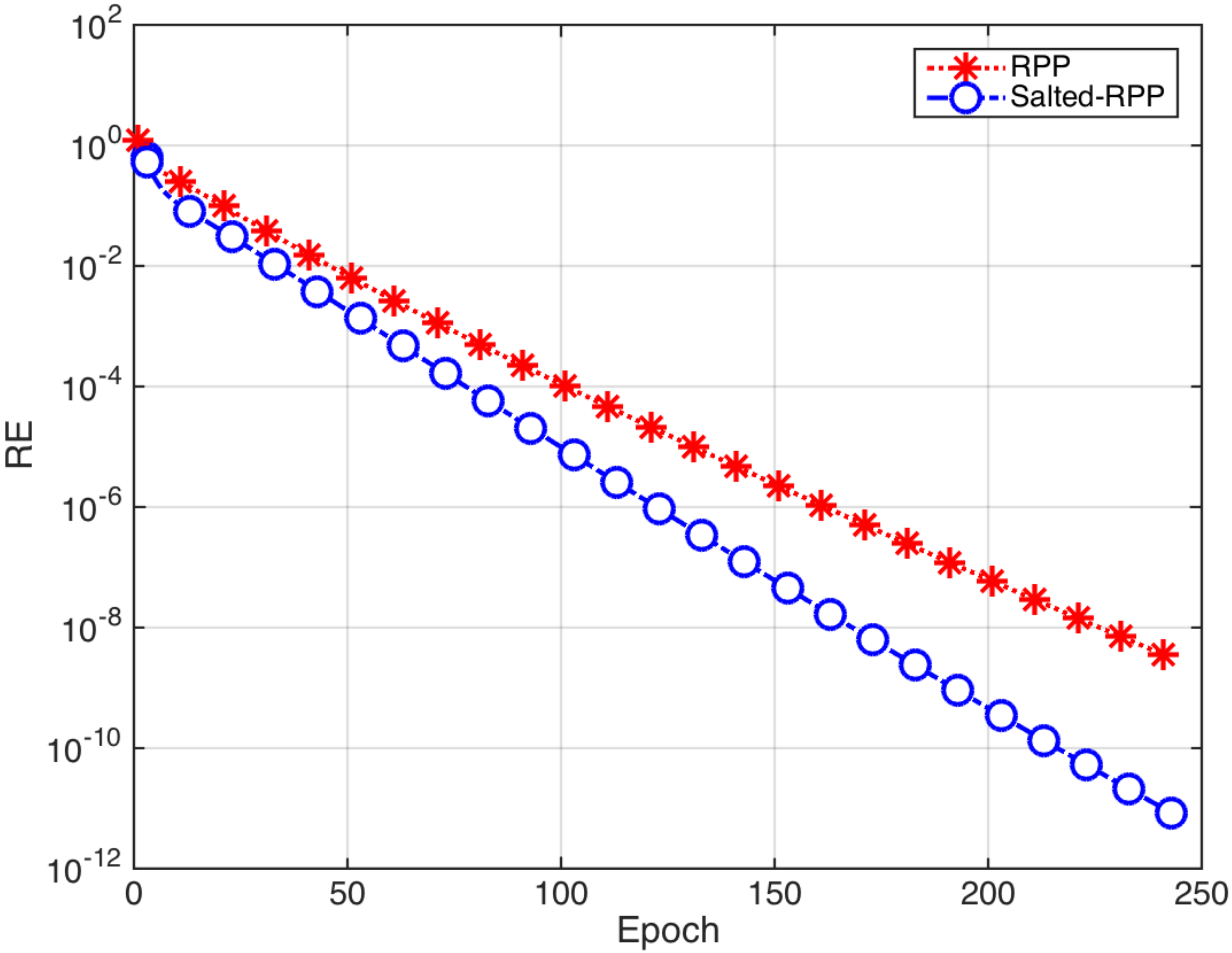}}
\caption{Adding salt noise speeds up the convergence.}
\label{figRPP}
\end{figure}
}

\commentout{
\begin{figure}
\centering
\subfigure[Phase error for RPP]{\includegraphics[width=8cm]{figs/phase-error-RPP.pdf}}
\subfigure[Phase error for salted RPP]{\includegraphics[width=8cm]{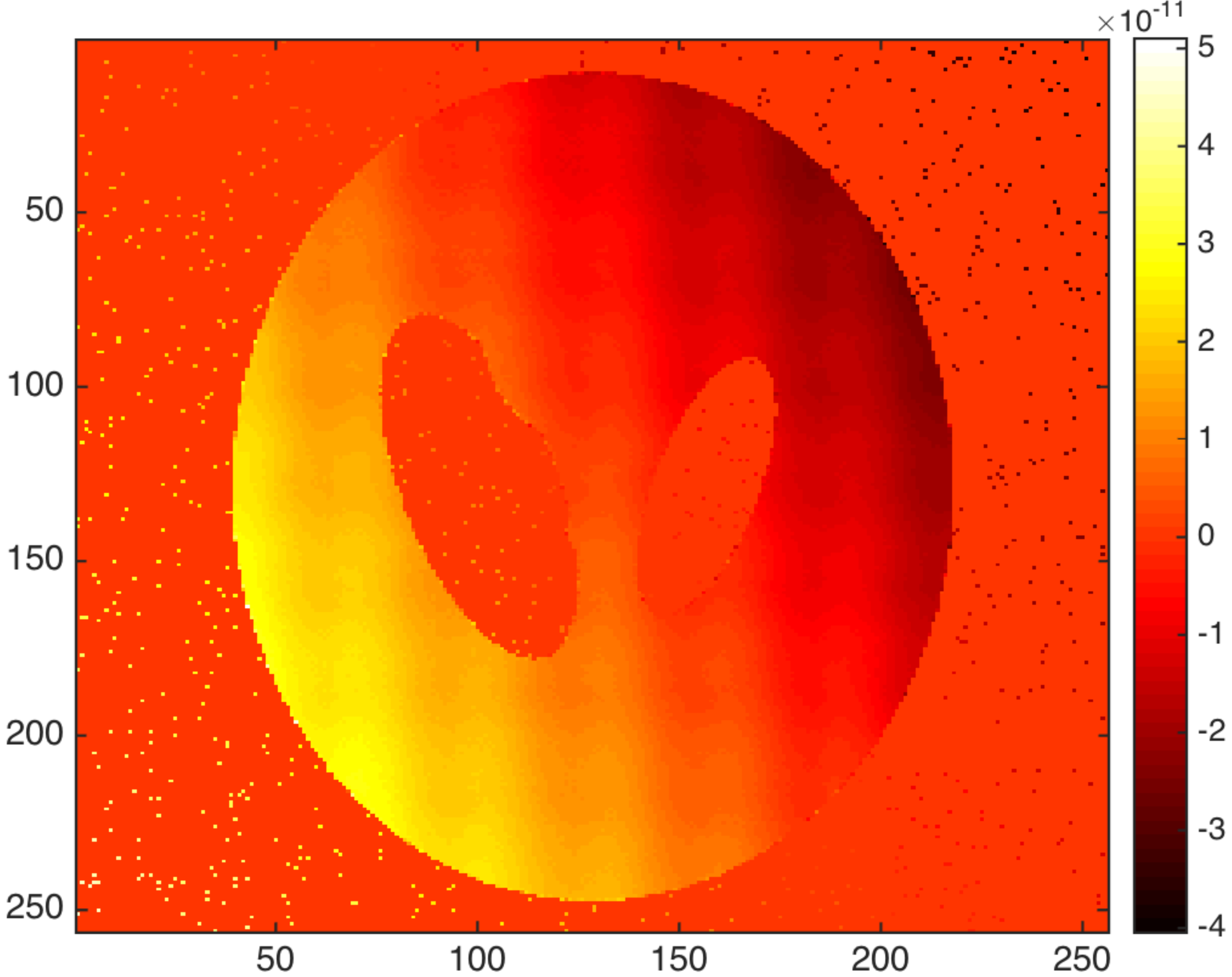}}
\caption{Phase difference between the true object and the reconstruction. Note the different scales in (b) and (c).}
\label{fig:phase}
\end{figure}
}
 Extensive area of dark pixels can pose a challenge to blind ptychographic reconstruction. 
As shown in Fig. \ref{figRPP}, the rate  ($r_o= 0.9345$) of convergence for RPP is larger than
that for CiB in Fig. \ref{fig1}.

The third test object is the salted RPP, the sum of
RPP and the salt noise (Fig. \ref{fig:RPP} (c)(d)). The salted noise is i.i.d, binomial random variables  with  probability $0.02$ to be a complex constant in the form of $a(1+\im), a\in \IR, $ and probability $0.98$ to be zero.  The salt noise reduces  the support looseness  without significantly changing the original image and makes the salted RPP more connected with respect to the ptychographic measurement.  We  explore the potential  of 
adding salt for enhancing ptychographic reconstruction. 

The added salt can be removed in postprocessing by denoising after the salted version is well reconstructed. 

Adding salt to RPP can speed up the convergence slightly, resulting in a smaller rate of 0.9115 and more accurate reconstruction. Fig. \ref{figRPP} and \ref{fig:phase} show the phase error of the reconstruction for
RPP and salted RPP, respectively. 

Another effect of the dark margin in RPP is that  the slope $\br$ of the linear phase ambiguity may be non-integers since the periodic boundary condition is ineffective on the dark margin. Adding salts can force $\br$ to be integers. For example,
$\br=(0,-1.231)$ for RPP while $\br=(-1,-1)$ for the salted RPP in Fig. \ref{figRPP} and \ref{fig:phase}. 

As mentioned above, the linear phase ambiguity is inherent to blind ptychographic reconstruction with either the periodic or dark-field (boundary value = 0) boundary
condition. Next  we demonstrate that the linear phase ambiguity is absent with the bright-field boundary condition. 
}

\commentout{
\begin{figure}
\centering
   \subfigure[Oversampling; PPC(-1,1,$\half$)]{\includegraphics[width=8cm]{report729/convDR_os_FBD_noLPS.png}}
    \subfigure[Non-oversampling; PPC ($-1,1,\frac{1}{5}$) ]{\includegraphics[width=8cm]{report729/convDR_nos_FBD_noLPS.png}}
 \caption{RE2 for Poisson-DRS  with the bright-field boundary condition at magnitude 255. }\label{fig:bright}
\end{figure}
}

\commentout{
\begin{figure}
\centering
\subfigure[Boundary value ]{\includegraphics[width=8cm]{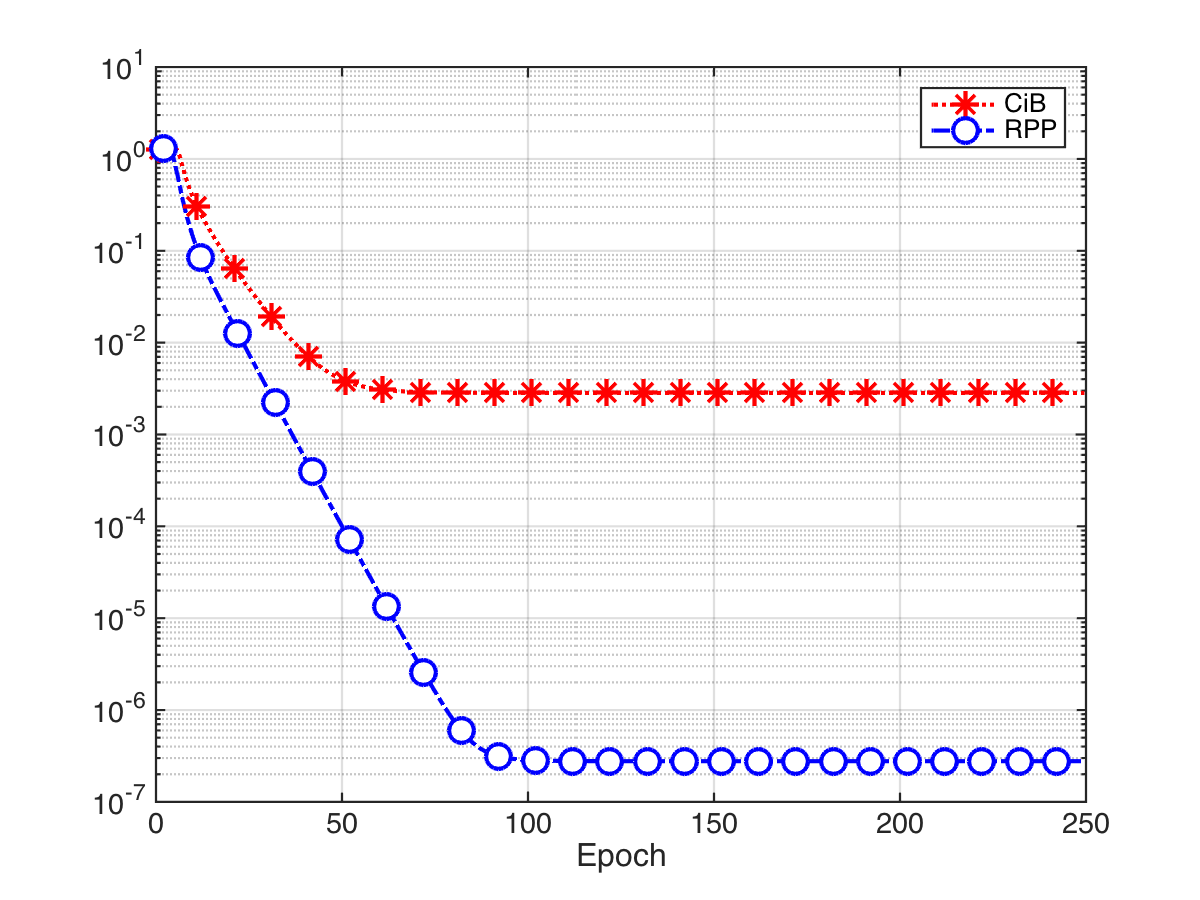}}
\subfigure[Boundary value ]{\includegraphics[width=8cm]{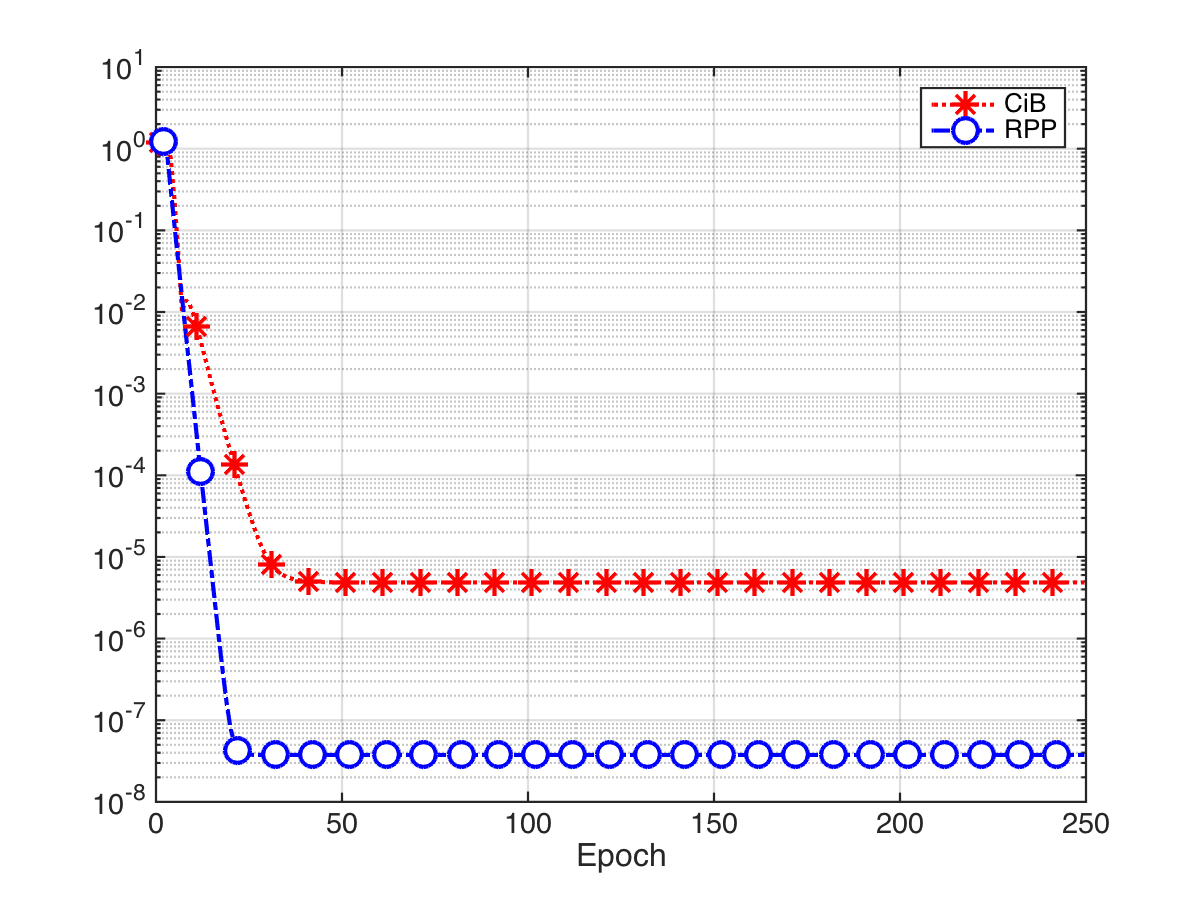}}
\caption{RE with PPC(-0.5, 0.5, $\half$) and the boundary value (a) 100 and (b) 255}
\label{fig:bright}
\end{figure}
}

\subsection{Linear phase ambiguity}

\begin{figure}[t]
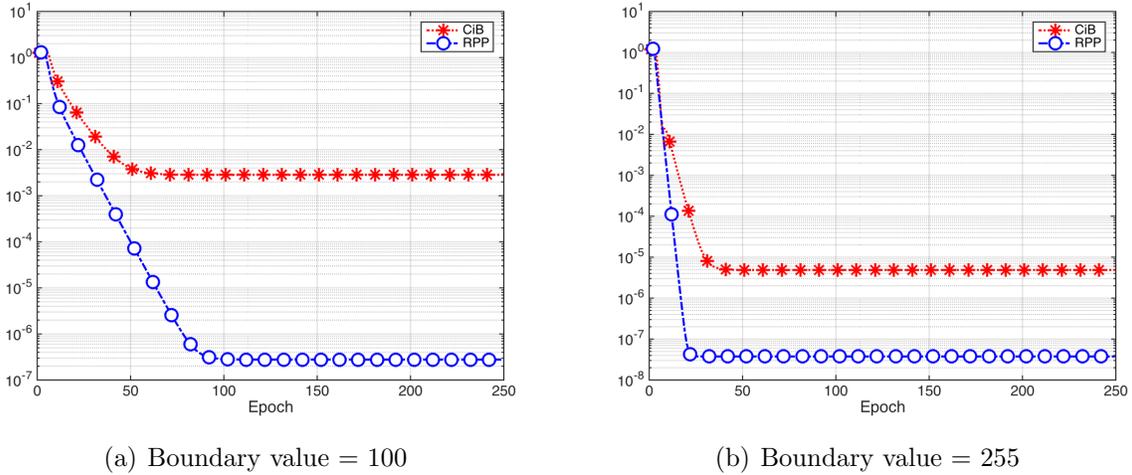

\centering
 \subfigure[Boundary value = 100 ]{\includegraphics[width=8cm]{figs/conv_LPS_CiBRPP_100inten_NoSRPP.png}}
 \subfigure[Boundary value = 255 ]{\includegraphics[width=8cm]{figs/conv_LPS_CiBRPP_255inten_NoSRPP.png}}
\caption{RE with PPC(-0.5, 0.5, $\half$) and the boundary value (a) 100 and (b) 255}
\label{fig:bright}
\end{figure}
\commentout{
 \begin{figure}
\centering
 \subfigure[Max iteration = 80]{\includegraphics[width=8cm]{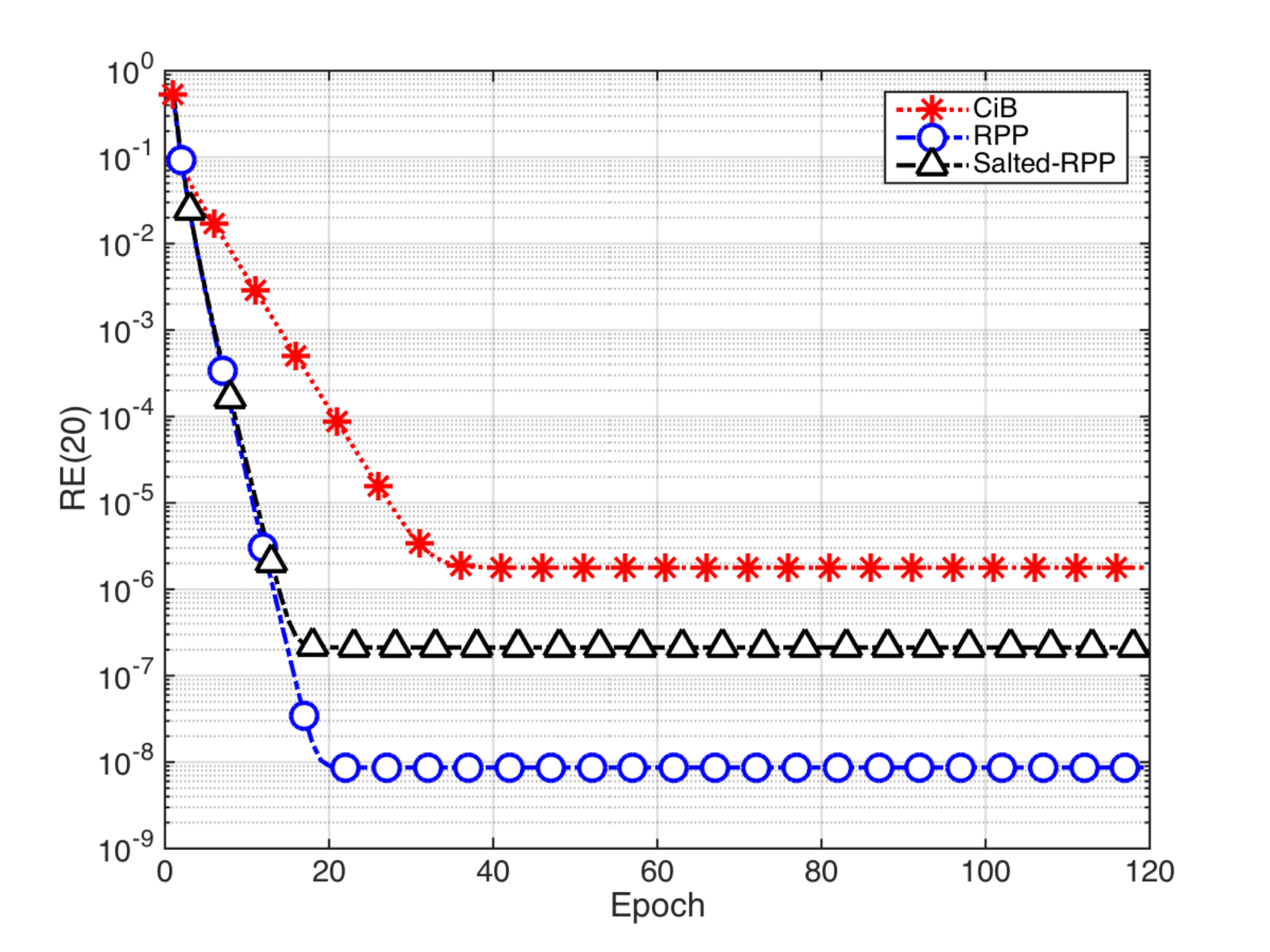}}
 \subfigure[Max iteration = 110]{\includegraphics[width=8cm]{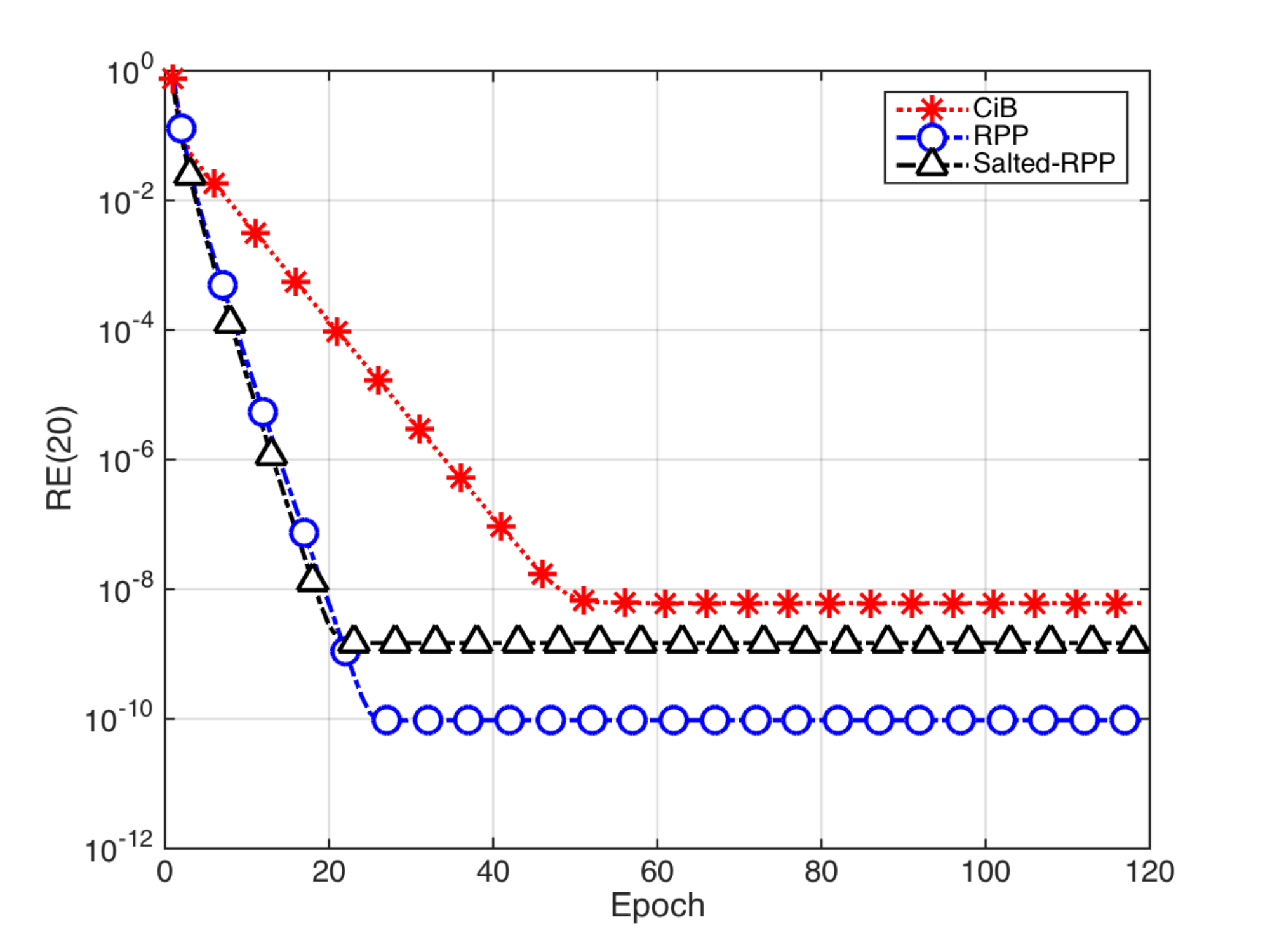}}
\caption{RE under  the boundary value = 255; PPC(-0.5, 0.5, 0)}
\label{fig10}
\end{figure}
}

To show that the linear phase ambiguity is absent under the bright-field boundary condition,
we test AMDRS with the initialization PPC$(-0.5, 0.5, \half)$ 
and use  a more stringent  error metric
\beq\label{RE2}
{\mbox{RE2}}(k)&= &\min_{\alpha\in \mathbb{C}}\frac{\| f(\mathbf{k}) -\alpha f_k(\mathbf{k})\|_2}{\|f\|_2}. 
\eeq
Note that PPC$(-0.5, 0.5, \half)$ violates the probe phase constraint $\Re(\bar \nu^0\odot \mu^0)>0$ allowed by the uniqueness theory.  The linear phase factor is introduced
in the initialization to test if it persists in the reconstruction. 

We also use the less tolerant  stopping rule
\begin{equation*}
\frac{\| |A_k A_k^\dagger u_{k}^{l}|-b\|_2-\||A_k A_k^\dagger u_{k}^{l+1}|-b\|_2}{\||A_k A_k^\dagger u_{k}^{l}|-b\|_2} \leq 10^{-5}
\end{equation*} 
for the inner loops with the maximum number of iteration  capped at 80, the rate of convergence accelerates. 

Fig. \ref{fig:bright} demonstrates the capability of the bright-field boundary condition to eliminate the linear phase ambiguity
as the stronger error metric \eqref{RE2} decays geometrically before settling down to the final level of accuracy. The higher boundary value  ($255$ in Fig. \ref{fig:bright}(a)) leads to faster convergence than the lower boundary value ($100$ in Fig. \ref{fig:bright}(b)). The final level of accuracy, however, depends on how accurately the inner loops for each epoch are solved. For example, increasing the maximum number of iteration from 80 to 110, significantly enhances the final accuracy
of reconstruction (not shown). 

\commentout{
\subsection{Comparison with rPIE}

In this section, we compare the performance of AMDRS in Fig. \ref{fig:bright} (a) with that of the regularized PIE (rPIE) \cite{rPIE17}, the most up-to-date version of ptychographic iterative engine (PIE).

Instead of using all the 64 diffraction patterns simultaneously to update the object and probe estimates, rPIE uses
one diffraction pattern at a time in a random order. As such rPIE is a variant of stochastic gradient descent. The potential benefits  include efficient memory use and a good speed boost by parallel computing resources. Unfortunately, rPIE
often fail to converge  with our set-up. 

To obtain reasonable results, we make two adjustments. First, we change  the random phases of RPP and salted RPP to i.i.d. uniform random variables in  $(-\frac{\pi}{2}, \frac{\pi}{2}$). Second, we use PPC($0,0,  \frac{1}{40}$)
for the probe initialization. 

There are three adjustable parameters in rPIE and we select these values $\alpha =0.95, \gamma_{\rm prb}=0.95, \gamma_{\rm obj}=0.9$ (see \cite{rPIE17} for definition). The order of updating small patches is randomly shuffled in each experiment. For each test image, we run 20 independent experiments and present the best run in Fig. \ref{fig:rPIE}.

\begin{figure}
\centering
{\includegraphics[width=8cm]{report729/conv_SGD_CiBRPP}}
 \caption{rPIE reconstruction with the probe initialization PPC ($0,0,\frac{1}{40}$).  }
 \label{fig:rPIE}
 \end{figure}
 }
\section{Conclusion}\label{sec:last}

We have presented reconstruction algorithms based on alternating minimization by Douglas-Rachford splitting. Our choices of the objective functions and the step size are informed by the uniqueness theory \cite{blind-ptych}, the Poisson noise model and the stability analysis. The confluence of the three considerations leads to the proposed scheme AMDRS.

Enabled by the PPC initialization method, AMDRS converges globally and geometrically 
in all our experiments except in the case of RPP with the dark-field boundary condition and the
initialization condition PPC(0, 0,  0.5) due to the extensive area of dark pixels.

The boundary condition can have a significant impact on the performance of numerical reconstruction.
For either the dark-field or bright-field condition, enforcing the boundary condition, whenever available, improves the rate of convergence. 
 
\commentout{
Further, we have shown that the bright-field boundary condition can remove
the linear phase ambiguity and accelerate convergence, and that adding salt noise can  improve the quality
of reconstruction of objects exhibiting extensive area of dark pixels provided that the added salt noise can
be further removed by postprocessing. 
}

\appendix
\commentout{
\section{Monotonicity}
\begin{thm}
The sequence of $\mathcal{L}(\mu,f)$ generated in algorithm \ref{alg: suedo algorithm AM} satisfies the following decreasing property. 

\begin{equation*}
\mathcal{L}(\mu_k,f_k) \geq \mathcal{L}(\mu_k,f_{k+1}) \geq \mathcal{L}(\mu_{k+1},f_{k+1})
\end{equation*}
\label{thm: residual decrease and its estimate}
\end{thm}
\begin{proof}{\ref{thm: residual decrease and its estimate}}
Let's review our algorithm \ref{alg: suedo algorithm AM}. 
It is easy to see the theorem \ref{thm: residual decrease and its estimate} can be derived directly from definition of the $f_{k+1}, \mu_{k+1}$.
\begin{equation*}
\begin{aligned}
&\mathcal{L}(\mu_k,f_k)\geq \min\limits_{f}\mathcal{L}(\mu_k,f) = \mathcal{L}(\mu_k,f_{k+1}) \\
&\mathcal{L}(\mu_k,f_{k+1})\geq \min\limits_\mu\mathcal{L}(\mu,f_{k+1})= \mathcal{L}(\mu_{k+1},f_{k+1}) \\
\end{aligned}
\end{equation*}
Next let's estimate the first order decrease in either image update \eqref{equ: AP image update} or mask update \eqref{equ: AP mask update} in epoch $k^{th}$ with gaussian log-likelihood assumption \eqref{equ: gaussian log-likelihood}. Denote the decrease of residual in image update in the $k^{th}$ epoch as 
\begin{equation*}
\Delta_k^x =  || |A_kf_k|-b||^2 -|||A_kf_{k+1}| - b ||^2 
\end{equation*}
We can use steepest decent and line search analysis to give an lowerbound  (first order) approximation of $\Delta_k^x$. Assume we have a perturbation $\Delta x = \epsilon u;\ ||u||=1$  near $f_k$. 
\begin{equation*}
\begin{aligned}
\Delta_k^x &\geq  || |A_kf_k| -b||^2  - || |A_k f_k+ \epsilon A_ku| - b ||^2\\
           & = || |A_kf_k| -b||^2 - || |A_kf_k|(1+\epsilon \Re(\Omega^*_{A_kf_k}A_ku)) -b||^2 \\
           & \approx -\epsilon(|A_kf_k|-b)^T\Re(\Omega^*_{A_kf_k}A_ku) \\
           & =-\epsilon(A_kf_k-\Omega_{A_kf_k}b)^T\Omega^*_{A_kf_k}(\Re(\Omega^*_{A_kf_k}A_k)\Re(u)-\Im(\Omega^*_{A_kf_k}A_k)\Im(u))
\end{aligned}
\end{equation*}
holds for all $u \in \mathbb{C}^{n\times n};\ ||u||=1$ and $\epsilon$ small. We can maximize the right hand side over $u$ to get an lower bound of $\Delta_k^x$ by Cauthy inequality and taking
\begin{equation*}
u^T=-(A_kf_k-\Omega_{A_kf_k}b)^T\Omega^*_{A_kf_k}(\Re(\Omega^*_{A_kf_k}A_k)-\imath\Im(\Omega^*_{A_kf_k}A_k)) / ||.||
\end{equation*} \newline
where $||.|| :=||(A_kf_k-\Omega_{A_kf_k}b)^T\Omega^*_{A_kf_k}(\Re(\Omega^*_{A_kf_k}A_k)-\imath\Im(\Omega^*_{A_kf_k}A_k))|| = ||(A_kf_k-\Omega_{A_kf_k}b)^T\Omega^*_{A_kf_k}(\bar{\Omega}^*_{A_kf_k}{A_\infty}^*_k)|| $
\begin{equation*}
\begin{aligned}
\Delta_k^x\geq& \epsilon(||(A_kf_k-\Omega_{A_kf_k}b)^T\Omega^*_{A_kf_k}\Re(\Omega^*_{A_kf_k}A_k)||^2+ ||(A_kf_k-\Omega_{A_kf_k}b)^T\Omega^*_{A_kf_k}\Im(\Omega^*_{A_kf_k}A_k)||^2)/||.|| \\
           = & \epsilon||(A_kf_k-\Omega_{A_kf_k}b)^T\Omega^*_{A_kf_k}(\Re(\Omega^*_{A_kf_k}A_k)-\imath\Im(\Omega^*_{A_kf_k}A_k)) ||^2/||.|| \\
           = & \epsilon||(A_kf_k-\Omega_{A_kf_k}b)^T\Omega^*_{A_kf_k}(\bar{\Omega}^*_{A_kf_k}{A_\infty}^\dagger_k) ||^2/||.||\\
           = & \epsilon|| A_k(A_kf_k-\Omega_{A_kf_k}b)||
\end{aligned}
\end{equation*} 
\newpage
By the same analysis, we can give an lower bound of decrease in mask update in the $k^{th}$ epoch,
\begin{equation*}
\Delta_k^\mu = \epsilon|| B_k(B_k\mu_k-\Omega_{B_k\mu_k}b)||
\end{equation*}
\end{proof}
}

\section{The Poisson versus Gaussian log-likelihood functions}\label{sec:likelihood}

Poisson distribution
\[
P(n)={\lambda^ne^{-\lambda}\over n!}
\]
Let $n=\lambda(1+\ep)$ where $\lamb\gg 1$ and $\ep\ll 1$.
Using  Stirling's formula 
\[
n!\sim \sqrt{2\pi n} e^{-n} n^n
\]
in the Poisson distribution, we obtain
\beqn
P(n)&\sim &{\lamb^{\lamb(1+\ep)}e^{-\lamb}\over \sqrt{2\pi} e^{-\lamb(1+\ep)} [\lamb(1+\ep)]^{\lamb(1+\ep)+1/2}}\\
&\sim &{1\over \sqrt{2\pi\lamb} e^{-\lamb\ep} (1+\ep)^{\lamb(1+\ep)+1/2}}.
\eeqn
By the asymptotic 
\[
(1+\ep)^{\lamb(1+\ep)+1/2}\sim e^{\lamb\ep+\lamb\ep^2/2}
\]
we have
\beq
\label{PG1}
P(n)\sim {e^{-\lamb\ep^2/2}\over \sqrt{2\pi \lamb}}={e^{-(n-\lambda)^2/(2\lamb)}\over \sqrt{2\pi\lamb}}. 
\eeq
Namely  in the low noise limit the Poisson noise is equivalent to the Gaussian noise of the mean $|y_0|^2$ and
 the variance equal to the intensity of the diffraction pattern.  
\commentout{and get the pdf
\[
 {e^{-\xi^2/(2\sigma^2)}\over \sqrt{2\pi\sigma^2}}
\]
}
The overall SNR can be tuned by varying the  signal energy $\|y_0\|^2$. 

The log-likelihood function for the right hand side of  \eqref{PG1} is 
\beq
\label{pg}
\sum_j \ln |y[j]| +{1\over 2} \lt|{b[j]\over |y[j]|}-|y[j]|\rt|^2,\quad b= \mbox{noisy diffraction pattern.}
\eeq
For small NSR and in the vicinity of $b$, we make the substitution 
\[
{\sqrt{b[j]}\over |y[j]|}\to 1,\quad \ln|y[j]|\to \ln\sqrt{b[j]}
\]
to obtain 
\beq
\label{pg2}
\mbox{const.}+ {1\over 2} \sum_j \lt|\sqrt{b[j]}-|y[j]|\rt|^2.
\eeq

\section{Proof of Proposition \ref{thm:stable}}\label{sec:stable}

Note that $A_\infty^\dagger u=A^\dagger_\infty x, B_\infty^\dagger v=B^\dagger_\infty y$ and 
\beqn
P_\infty u=P_\infty x, &&P^\perp_\infty u=- P^\perp_\infty x\\
\quad Q_\infty v=Q_\infty y, && Q^\perp_\infty v=-Q^\perp_\infty y,
\eeqn
 Hence we also have $
u =R_\infty x,\quad v =S_\infty y. $

There are two ways to express \eqref{sa}-\eqref{sb} in terms of $y,x$: 
First, by reorganizing, 
\beq
\label{32}
b\odot\sgn(x)&=&P_\infty x-\rho P^\perp_\infty x\\
b\odot\sgn(y)&=&Q_\infty y-{\rho} Q^\perp_\infty  y,\label{33}
\eeq
and, second, after  the reflections  $R_\infty$ and $S_\infty$, respectively, 
\beq
\label{30}
P_\infty x+\rho P^\perp_\infty x & = & R_\infty \lt(b\odot \sgn{(x)}\rt) \\
Q_\infty y+{\rho} Q^\perp_\infty  y&=& S_\infty \lt(b\odot \sgn{(y)}\rt).\label{31}
\eeq

From \eqref{32}-\eqref{33} we obtain by the Pythogoras theorem 
\beq
\|P_\infty x\|^2_2+\rho^2 \|P^\perp_\infty x\|_2^2=\|b\|_2^2=\|Q_\infty y\|^2_2+\rho^2 \|Q^\perp_\infty y\|_2^2.
\eeq

If $\|P^\perp_\infty x\|_2=0$, then $x=P_\infty x=b\odot \sgn(x)$ by \eqref{32}, implying
$b=|A_\infty f_\infty|$ for $f_\infty :=A_\infty^\dagger u=A^\dagger x$. 
Likewise, if $\|Q^\perp_\infty y\|_2=0$, then $y=Q_\infty y=b\odot \sgn(y)$ by \eqref{33}, implying 
$b=|B_\infty \mu_\infty|$ for $\mu_\infty :=B_\infty^\dagger v=B^\dagger_\infty y$. In other words, $(f_\infty,\mu_\infty)$ solve
the blind ptychography problem. 

We now prove Proposition \ref{thm:stable} by contradiction. For our purpose, it suffices to focus on \eqref{32}.  

Suppose $\|P_\infty x\|_2< \|b\|_2$ (or equivalently $\|P^\perp_\infty x\|_2\neq 0$). 
Applying $\Omega^*$  and rewriting \eqref{30} we have
\beqn
CC^\dagger |x|+ \rho(|x|-CC^\dagger |x|)&=& 2CC^\dagger b-b
\eeqn
On the other hand, applying the operator $C$ on \eqref{30} we have 
\begin{equation*}
CC^\dagger |x| = CC^\dagger b.
\end{equation*}
Combining the above two relations, we obtain
\beq
\label{38}
CC^\dagger |x| = CC^\dagger b={\rho\over 1+\rho}|x|+{1\over 1+\rho} b.\label{key}
\eeq

Using \eqref{key} we now show that $\|J_A(\eta)\| > \|\eta\|$ if $\rho\ge 1$ and $\eta=\imath (\alpha |x|+\beta b)$ where $\alpha$ and
$\beta$ are two positive constants such that $\rho \beta >\alpha$. The choice of a purely imaginary $\eta$ is to nullify $\Re\big(\eta-2CC^\dagger \eta \big)$
by means of \eqref{38}.  By \eqref{38}, $\eta$ can be written as 
\beqn
\eta 
 &=&\imath(\alpha+{\alpha \over \rho}) CC^\dagger|x|+\imath (\beta-{\alpha\over \rho})b
 \eeqn

After some algebra with \eqref{38'} and \eqref{38}, we arrive at
\[
J_A(\eta)={\imath\over 1+\rho} \lt\{\lt[{\rho(\rho-1)\over \rho+1}(\alpha+\beta)+\alpha\rt]|x|+2\lt[{2\beta\rho\over \rho+1}+{\alpha(\rho-1)\over \rho+1}\rt]b+\lt[{2\alpha\over \rho+1} + {(1-\rho)\beta\over 1+\rho}\rt]{b^2\over |x|}\rt\}
\]

For the rest of our analysis the case of $\rho=1$ is particularly transparent.  

\commentout{
For $\rho=0$,
\beqn
J_A(\eta)&=& \imath\lt[\alpha |x|-2\alpha b+(2\alpha+\beta) {b^2\over |x|}\rt]
\ge \imath (\alpha+\beta) (2b-|x|)
\eeqn
 by the quadratic inequality. 
 }
 
 For $\rho=1$,
 \[
 J_A(\eta)={\imath\over 2}\lt[\alpha|x|+2\beta b+ \alpha b^2/|x|\rt]. 
 \]
 By the quadratic inequality, $\alpha |x|+\alpha b^2/|x|\ge 2\alpha b$ and hence 
 $\|J_A(\eta)\|_2 \ge (\alpha+\beta)\|b\|_2.$ On the other hand, by \eqref{40}, $
 \eta=\imath 2\alpha CC^\dagger |x|+\imath (\beta-\alpha)b$
 and
 hence
 \beqn
 \|\eta\|_2&\le& 2\alpha \|CC^\dagger x\|_2+(\beta-\alpha)\|b\|_2
 \eeqn
 where $\alpha<\beta$. 
 Therefore, as a consequence of $\|CC^\dagger |x|\|_2 = \|P_\infty x\|_2< \|b\|_2$,  the desired result follows:
  \[
  \|\eta\|_2 < (\alpha+\beta)\|b\|_2\le \|J_A(\eta)\|_2.
  \]
 
Finally, for any linearly stable fixed point of Gaussian DRS, we now know from the above analysis that $\|P^\perp_\infty x\|_2=0$ and hence $x=P_\infty x$ which, combined with \eqref{32}, yields
the sought after statement \eqref{key2}.

\section{Proof of Proposition \ref{thm:stable2}}\label{sec:stable2}

As we have seen in the proof of Proposition \ref{thm:stable} that
any  blind ptychography solution satisfies $
x=P_\infty x=b\odot \sgn(x)
$
and hence by \eqref{38'}
\beq
\label{38''}
J_A(\eta) &= & CC^\dagger\eta+ {1\over 1+\rho}\Re\big(\eta-2CC^\dagger \eta \big)\\
&=&  \imath \Im\lt(CC^\dagger\eta\rt)+ \Re\lt(CC^\dagger \eta\rt)+{1\over 1+\rho}\Re\big(\eta-2CC^\dagger \eta \big).\nn
\eeq
We now show that $\|J_A(\eta)\|_2\le \|\eta\|_2$ for all $\eta$. The case for $J_B$ can be similarly analyzed. 

To proceed, we shall write $CC^\dagger=GG^*$ where $G$ is an isometry. 
This can  be done for any matrix $C$ via the QR decomposition. 

According to Proposition 5.9 and Corollary 5.10 in \cite{FDR}, $GG^*$ 
can be block-diagonalized into one $(N-2n^2)\times (N-2n^2)$ zero block and $2n^2$ $2\times 2$ blocks
\beq\label{block}
\lt[ \begin{matrix}
\lambda_k^2& \lambda_k\lambda_{2n^2+1-k}\\
\lambda_k\lambda_{2n^2+1-k}&\lambda_{2n^2+1-k}^2
\end{matrix}\rt],\quad k=1,2,\cdots, 2n^2
\eeq
in the orthonormal basis $\{\eta_k, \imath \eta_{2n^2+1-k}: k=1,2,\cdots, 2n^2\}$
where $ \eta_k\in \IR^N$ are the right singular vectors, corresponding to the singular values $\lamb_k$,  of
 \begin{equation} \label{Bv}
\left[
\begin{matrix}
 \Re[K^*]\\
  \Im[K^*] 
\end{matrix}
\right]\in \IR^{2n^2,N}.
\end{equation}
Moreover, the complete set of singular values satisfy  
\beq
\label{40}
&1=\lamb_1\ge \lamb_2\ge \cdots\ge \lamb_{2n^2}= \lamb_{2n^2+1}=\cdots=\lamb_{N}=0&\\
&\lamb_k^2+\lamb_{2n^2+1-k}^2=1.&\label{41}
\eeq

In view of the block-diagonal nature of $GG^*$, we shall analyze $J_A(\eta)$ in the 2-dim spaces spanned by the orthonormal basis $\{\eta_k, \imath \eta_{2n^2+1-k}\}$ one $k$ at a time. Note that the first basis vector $\eta_k$ is real and the second  $\imath \eta_{2n^2+1-k}$ is purely imaginary. 

For any fixed $k$, we write $\eta=(z_1,z_2)^T\in \IC^2$ and obtain 
\beq\label{50}
J_A(\eta) &=&\lt(\lamb_k^2z_1+\lamb_k\lamb_{2n^2+1-k}z_2\rt)\eta_k+
\lt(\lamb_k\lamb_{2n^2+1-k}z_1+\lamb^2_{2n^2+1-k}z_2\rt)\imath\eta_{2n^2+1-k}\\
&&+{1\over 1+\rho}\lt[
(1-2\lamb_k^2)\Re(z_1)-2\lamb_k\lamb_{2n^2+1-k} \Re(z_2)\rt]\eta_k \nn\\
&&
+ {1\over 1+\rho}\lt[2\lamb_k\lamb_{2n^2+1-k} \Im(z_1)-(1-2\lamb_{2n^2+1-k}^2)\Im(z_2) \rt]\eta_{2n^2+1-k}. \nn
\eeq
Next we treat \eqref{50} as a linear function of $\Re(z_1), \Re(z_2), \Im(z_1), \Im(z_2)$ with real coefficients in the basis $\{\eta_k,\imath\eta_{2n^2+1-k}, \imath \eta_k, \eta_{2n^2+1-k}\}$ and represent $J_A$ by
a $4\times 4$ matrix which is block-diagonalized into two $2\times 2$ blocks:
\beq\label{51}
\lt[\begin{matrix} {1\over 1+\rho}+{\rho-1\over\rho+1} \lamb_k^2 & {\rho-1\over \rho+1}\lamb_k\lamb_{2n^2+1-k}\\
\lamb_k\lamb_{2n^2+1-k}& \lamb_{2n^2+1-k}^2
\end{matrix}\rt],\quad
\lt[\begin{matrix}
\lamb_k^2 &\lamb_k\lamb_{2n^2+1-k}\\
-{\rho-1\over \rho+1}\lamb_k\lamb_{2n^2+1-k} &-{1\over \rho+1}-{\rho-1\over\rho+1} \lamb_{2n^2+1-k}^2
\end{matrix}\rt].
\eeq
with the former of \eqref{51} acting on $\Re(z_1), \Re(z_2)$ and 
the latter acting on $\Im(z_1), \Im(z_2)$. 
In view of their similar structure, it suffices to focus on the former.

Calculation of the eigenvalues  with \eqref{40}-\eqref{41} gives 
\beq
\label{52}
{1\over 2(\rho+1)} \lt[{\rho} +{2} (1-\lamb_k^2)\pm \sqrt{\rho^2-4\lamb_k^2+4\lamb_k^4}\rt]
\eeq
which, with the + sign,  equals 1 at $k=2n^2$ (recall $\lamb_{2n^2}=0$).
Next we  show that 1 is the largest eigenvalue among all $k$ and $\rho\in [0,\infty)$.

Note that the radical in \eqref{52} is real for any $\lamb_k\in [0,1]$ iff $\rho\ge 1$.
Hence, for $\rho\ge 1$, the maximum eigenvalue is 1 and occurs at $k=2n^2$. 
In particular, \eqref{52} becomes
\beq
\label{53}
{1\over 4} \lt[1+2(1-\lamb_k^2)\pm |1-2\lamb_k^2|\rt]\quad \mbox{for} \quad\rho=1,  
\eeq
which achieves the maximum value 1 at $k=2n^2$ and the second largest
eigenvalue $1-\lamb_{2n^2-1}^2$ at $k=2n^2-1$. 

For $\rho<1$ and those $k$'s giving rise to some purely imaginary radical in \eqref{52},
we calculate the modulus of \eqref{52} and obtain the upper bound
\beq
\label{55}
\sqrt{1-\lamb_k^2\over 1+\rho}\le \sqrt{1\over 1+\rho}\le 1,\quad \forall k.
\eeq

In view of \eqref{52} and the leftmost term in \eqref{55}, it is clear that the second largest (in modulus) eigenvalue also occurs at $k=2n^2-1$. 

According Proposition 5.6 in \cite{FDR}, $\eta_1=b$ and hence $\arg\min_{\|\eta\|=1}|J_A(\eta)|$
occurs in the subspace spanned by $\{\eta_{2n^2}, \imath b\}$ where $\eta_{2n^2}$ is a real-valued null vector of \eqref{Bv}. It is easy to check that a real null vector of  \eqref{Bv} is also a
null vector of $K^*$ and hence of $P_\infty$. Since $x$ is in the range of $P_\infty$,
we conclude that $\arg\min_{\|\eta\|=1}|J_A(\eta)|$ contains $\pm \imath b/\|b\|$ and possibly elsewhere
since we do not know if $\lamb_k>0$ for $k>2n^2$ without additional conditions.

The proof is complete.


\section*{Acknowledgements} The research of A. F. is supported by  the US National Science Foundation  grant DMS-1413373. AF thanks National Center for
Theoretical Sciences (NCTS), Taiwan,   where  the present work was carried out, for the hospitality  during his visits  in June and August 2018.


\begin{thebibliography}{plain}
\bibitem{keyhole}

B. Abbey, K. A. Nugent, G. J. Williams, J.N. Clark, A.G. Peele , M.A. Pfeiffer, M. De Jonge \& I. McNulty, ``Keyhole coherent diffractive imaging." {\em Nat. Phys.} {\bf 4} (2008) 394-398.



  





\bibitem{Poisson2}
L. Bian, J. Suo, J. Chung, X. Ou, C. Yang, F. Chen, and Q. Dai, ``Fourier ptychographic reconstruction using Poisson
maximum likelihood and truncated Wirtinger gradient," {\em Sci. Rep. } {\bf 6} (2016), 27384.





\bibitem{overlap}%
O. Bunk, M. Dierolf, S. Kynde, I. Johnson, O. Marti, F. Pfeiffer, ``Influence of the overlap parameter on the convergence of the ptychographical iterative engine," {\em  Ultramicroscopy} {\bf 108} (5) (2008) 481-487.




\bibitem{CDI}%
H.N. Chapman \& K.A. Nugent, ``Coherent lensless X-ray imaging," {\em Nat. Photon.} {\bf 4} (2010) 833-839.




\bibitem{FDR}%
P. Chen and  A. Fannjiang, ``Phase retrieval with a single mask 
by Douglas-Rachford algorithms," {\em Appl. Comput. Harmon. Anal.} {\bf  44} (2018), 665-699. 

\bibitem{ptych-unique}
P. Chen and A. Fannjiang, `` Coded-aperture ptychography: uniqueness and reconstruction", {\em Inverse Problems} {\bf 34 }(2018) 025003.


\bibitem{AP-phasing}%
P. Chen, A. Fannjiang and G. Liu, ``Phase retrieval with one or two coded diffraction patterns by alternating projection with the null initialization," {\em J. Fourier Anal. Appl.} {\bf 24} (2018), 719-758. 


\bibitem{ptycho10}%
M. Dierolf, A. Menzel, P. Thibault, P. Schneider, C. M. Kewish, R. Wepf, O. Bunk, and F. Pfeiffer, `` Ptychographic x-ray computed tomography at the nanoscale,'' {\em Nature} {\bf 467} (2010), 436-439.


\commentout{\bibitem{FAK}
C. Falldorf, M. Agour, C. v. Kopylow and R. B. Bergmann,
``Phase retrieval by means of a spatial light modulator in the Fourier domain of an imaging system," {\em Appl. Opt.} {\bf 49}, 1826-1830 (2010). 
}


\bibitem{DR}
J. Douglas and H.H. Rachford, "On the numerical solution of heat conduction problems in two and three space variables," {\em Trans.  Am.  Math.  Soc.}{\bf 82} (1956), 421-439.


\bibitem{Elser}
V. Elser, 
``Phase retrieval by iterated projections," 
{\em J. Opt. Soc. Am. A}  {\bf 20} (2003), 40-55.

\bibitem{EB92}
J.	Eckstein	and D.P. Bertsekas, ``On the Douglas-Rachford splitting method and the proximal point algorithm for maximal monotone operators,"  {\em Math. Program. A} {\bf 55} (1992), 293-318.





\bibitem{blind-ptych}

A. Fannjiang \& P. Chen, ``Blind ptychography: uniqueness and ambiguities,"  {\tt arXiv:1806.02674}. 



\bibitem{pum}%
A. Fannjiang and W. Liao, ``Fourier phasing with phase-uncertain mask," {\em Inverse Problems} 
{\bf 29} (2013) 125001.



\bibitem{PIE104}%
H.M.L. Faulkner and J.M. Rodenburg, ``Movable aperture lensless transmission microscopy: A novel phase retrieval algorithm," {\em Phys. Rev. Lett.}{\bf 93} (2004), 023903. 

\bibitem{PIE05}%
H.M.L. Faulkner and J.M. Rodenburg,
``Error tolerance of an iterative phase retrieval algorithm for moveable illumination microscopy,"
{\em Ultramicroscopy} {\bf 103:2} (2005), 153-164. 


\bibitem{GM76}
D. Gabay and B.  Mercier. A dual algorithm for the solution of nonlinear variational problems via finite element approximation. {\em Computers \& Mathematics with Applications} {\bf 2(1)} (1976), 17-40.

\bibitem{EM-ptych1}
S. Gao, P. Wang,  F.  Zhang, G. T. Martinez, P. D. Nellist, X. Pan \& A. I. Kirkland,
``Electron ptychographic microscopy for three-dimensional imaging,"
{\em Nat. Comm.} {\bf 18}(2017) 163. 

\bibitem{GM75}
R. Glowinski and A. Marroco. ``Sur l'approximation, par \'el\'ements finis d'ordre un, et la r\'esolution, par p\'enalisation-dualit\'e d'une classe de probl\'emes de dirichlet non lin\'eaires," 
{\em ESAIM: Mathematical Modelling and Numerical Analysis}, {\bf 9}(1975), 41-76.

\bibitem{noise}
P. Godard, M. Allain, V. Chamard, and J. Rodenburg, ``Noise models for low counting rate coherent diffraction imaging," {\em  Opt. Express} {\bf 20} (2012), 25914-25934.

\bibitem{Fie08}%
M. Guizar-Sicairos,  J.R. \& Fienup,  ``Phase retrieval with transverse translation diversity: a nonlinear optimization approach." {\em Opt. Express} {\bf 16} (2008), 7264-7278.


\bibitem{Hesse}%
R. Hesse, D. R. Luke, S. Sabach, and M.K. Tam,
``Proximal heterogeneous block implicit-explicit method and application to blind ptychographic diffraction imaging," {\em SIAM J. Imag. Sci.} {\bf 8} (2015) pp. 426-457. 



\commentout{
\bibitem{Hoppe1}%
W. Hoppe, ``Beugung im inhomogenen Primärstrahlwellenfeld. I. Prinzip einer Phasenmessung von Elektronenbeungungsinterferenzen". {\em Acta Cryst.. A.} {\bf 25} (4) (1969) 495.

\bibitem{Hoppe2}%
W. Hoppe, ``Beugung im inhomogenen Primärstrahlwellenfeld. III. Amplituden- und Phasenbestimmung bei unperiodischen Objekten", {\em Acta Cryst. A}. {\bf 25} (4) (1969) 508.
}

\bibitem{Horisaki1}
R. Horisaki, R. Egami \& J. Tanida, ``Single-shot phase imaging with randomized light (SPIRaL)". {\em Opt. Express} {\bf 24} (2016), 3765-3773.

\bibitem{scan1}
X. Huang, H. Yang, R. Harder, Y. Hwu, I.K. Robinson \& Y.S. Chu,``Optimization of overlap uniformness for ptychography," {\em Opt. Express} {\bf 22} (2014), 12634-12644.



 

\bibitem{EM-ptych2}
Y. Jiang,  Z.  Chen,  Y. Han, P. Deb, H. Gao, S. Xie, P. Purohit, M. W. Tate, J. Park, S. M. Gruner, V.  Elser \& D. A. Muller
``Electron ptychography of 2D materials to deep sub-angstrom resolution," {\em Nature} {\bf 559} (2018) 343-349. 


\bibitem{mask2} C. Kohler, F. Zhang and W. Osten, 
``Characterization of a spatial light modulator and its application in phase retrieval,"
{\em Appl. Opt.} {\bf 48} (2009) 4003-4008.

\bibitem{noise3} A.P. Konijnenberg, W.M.J. Coene and H.P. Urbach,
``Model-independent noise-robust extension of ptychography,"
{\em Opt. Exp.}{\bf 26} (2018) 5857-5874.


\bibitem{FPM15}
C. Kuang, Y. Ma, R. Zhou, J. Lee, G. Barbastathis, R. R. Dasari, Z. Yaqoob \& P.T.C. So, ``Digital micromirror device-based laser-illumination Fourier ptychographic microscopy," {\em Opt. Exp. }{\bf 23}(2015), 26999-27010. 

\bibitem{LP} G. Li \& T. K. Pong,
``DouglasÐRachford splitting for nonconvex optimization with application to nonconvex feasibility problems,"
{\em Math. Program.} {\bf A 159} (2016), 371-401

\bibitem{LM79}
P.-L. Lions and B. Mercier,``Splitting algorithms for the sum of two nonlinear operators," {\em SIAM  J. Num.  Anal.} {\bf 16} (1979), 964-979.



\bibitem{supres-PIE}%
A. M. Maiden, M. J. Humphry, F. Zhang and J. M. Rodenburg, ``Superresolution imaging via ptychography,"
{\em J. Opt. Soc. Am. A} {\bf 28} (2011), 604-612. 

\bibitem{rPIE17}
A. M. Maiden, D. Johnson and P. Li, 
``Further improvements to the ptychographical iterative engine," {\em Optica}
{\bf 4} (2017), 736-745.

\bibitem{ptycho-rpi}%
A.M. Maiden,	 G.R. Morrison,	 B. Kaulich,	 A. Gianoncelli	 \& J.M. Rodenburg,
``Soft X-ray spectromicroscopy using ptychography with randomly phased illumination,"
{\em Nat. Commun.} {\bf 4} (2013), 1669. 


\bibitem{ePIE09}%
A.M. Maiden \& J.M. Rodenburg, ``An improved ptychographical phase retrieval algorithm for diffractive imaging," {\em Ultramicroscopy} {\bf 109} (2009), 1256-1262.

\bibitem{ePIE10}%
A. M. Maiden, J. M. Rodenburg and M. J. Humphry,
``Optical ptychography: a practical implementation with useful resolution," {\em Opt. Lett.}
{\bf 35} (2010), 2585-2587. 


\bibitem{zone-plate}
G.R. Morrison, F. Zhang, A. Gianoncelli and I.K. Robinson,
``X-ray ptychography using randomized zone plates," {\em Opt. Exp.} {\bf 26} (2018) 14915-14927.

\bibitem{parallel}
Y. S. G. Nashed, D. J. Vine, T. Peterka, J. Deng,
R. Ross  and C. Jacobsen, 
``Parallel ptychographic reconstruction," {\em Opt. Express} {\bf 22} (2014) 32082-32097. 

\bibitem{EM0}
P.D. Nellist, B.C. McCallum \& J.M. Rodenburg, ``Resolution beyond the information limit in transmission electron microscopy, " {\em Nature} {\bf 374} (1995) 630-632. 

\bibitem{EM1}
P.D. Nellist and J.M. Rodenburg, `` Electron ptychography. I. Experimental demonstration beyond the conventional resolution limits," {\em Acta Cryst. A} {\bf 54} (1998), 49-60. 

\bibitem{Nugent}
K.A. Nugent, ``Coherent methods in the X-ray sciences, " {\em Adv. Phys.}{\bf 59} (2010) 1-99. 

\bibitem{Yang14}
X. Ou, G. Zheng and C. Yang, ``Embedded pupil function recovery for Fourier ptychographic microscopy,"
{\em Opt. Exp.} {\bf 22} (2014) 4960-4972. 



\bibitem{random-aperture}
X. Peng, G.J. Ruane, M.B. Quadrelli \& G.A. Swartzlander, `` Randomized apertures: high resolution imaging in far field," {\em Opt. Express} {\bf 25} (2017) 296187. 

\bibitem{Pfeiffer}
F. Pfeiffer, ``X-ray ptychography," {\em Nat. Photon.} {\bf 12} (2017) 9-17. 


\bibitem{Rod08}
J.M. Rodenburg, ``Ptychography and related diffractive imaging methods," {\em Adv. Imaging  Electron Phys.} {\bf 150} (2008) 87-184. 


\bibitem{PIE204}%
J.M.  Rodenburg and H.M.L.  Faulkner, ``A phase retrieval algorithm for shifting illumination". {\em Appl. Phys.  Lett.} {\bf 85} (2004), 4795.


\bibitem{RCM}
M.H. Seaberg, A. d'Aspremont \& J.J. Turner, ``Coherent diffractive imaging using randomly coded masks," {\em Appl. Phys. Lett.} {\bf 107} (2015) 231103.

\bibitem{FROG}
P. Sidorenko, O. Lahav, Z. Avnat \& O. Cohen,
``Ptychographic reconstruction algorithm for frequency-resolved optical gating: super-resolution and supreme robustness," {\em Optica} {\bf 3} (2016) 1320-1330.

\bibitem{ptych-theory}
J.C. Silva and A. Menzel, 
``Elementary signals in ptychography," {\em Opt. Exp.} {\bf 23} (2015) 33812-33821. 




\bibitem{diffuser}
M. Stockmar, P. Cloetens, I. Zanette, B. Enders, M. Dierolf, F. Pfeiffer, and P. Thibault,``Near-field ptychography: phase retrieval for inline holography using a structured illumination," {\em Sci. Rep.} {\bf 3} (2013), 1927.


\bibitem{random-coding}
D. Sylman, V. Mic\'o, J. Garc'a \& Z. Zalevsky,
``Random angular coding for superresolved imaging,"
{\em Appl. Opt.}{\bf 49} (2010), 4874-4882. 

\bibitem{probe09}%
P. Thibault, M. Dierolf, O. Bunk, A. Menzel, F. Pfeiffer, 
``Probe retrieval in ptychographic coherent diffractive imaging,"
{\em Ultramicroscopy}
{\bf 109} (2009),  338-343.

\bibitem{DM08}%
P. Thibault, M. Dierolf, A. Menzel, O. Bunk, C. David, F. Pfeiffer, ``High-resolution scanning X-ray diffraction microscopy", {\em Science} {\bf 321} (2008),  379-382.

\bibitem{ML12}
P. Thibault and M. Guizar-Sicairos, `` Maximum-likelihood refinement for coherent diffractive imaging". {\em New J. Phys.} {\bf 14} (2012), 063004.

\bibitem{Tian15}
L. Tian, Z. Liu, L-H Yeh, M. Chen, J. Zhong, L. Waller, ``Computational illumination for high-speed in vitro Fourier ptychographic microscopy," {\em Optica} {\bf 2} (2015) 904-911. 

\bibitem{terahz}
L. Valzania, T. Feurer, P. Zolliker \& E. Hack,
``Terahertz ptychography," {\em Opt. Lett.} {\bf 43} (2018), 543-546. 

 
 \bibitem{ADM}
Z. Wen, C. Yang, X. Liu and S. Marchesini,  ``Alternating direction methods for classical and ptychographic phase retrieval," {\em  Inverse Problems} {\bf 28} (2012),  115010.


\bibitem{Waller15}
L. Yeh, J. Dong, J. Zhong, L.Tian, M. Chen, G. Tang, M. Soltanolkotabi,  and L.  Waller,
``Experimental robustness of Fourier  ptychography phase retrieval algorithms," {\em Optics Express}
{\bf 23} (2015) 33214-33240.


\bibitem{Fucai2}
F. Zhang, B. Chen, G. R. Morrison, J. Vila-Comamala, M. Guizar-Sicairos
\& I. K. Robinson, ``Phase retrieval by coherent modulation imaging,"
{\em Nat. Comm.} {\bf 7} (2016):13367. 


\bibitem{Fucai1}
F. Zhang, G. Pedrini \& W. Osten,
``Phase retrieval of arbitrary complex-valued fields through aperture-plane modulation," {\em Phys. Rev. A} {\bf 75} (2007), 043805.  


 
\bibitem{noise2}
Y. Zhang, P. Song, Q. Dai, ``Fourier ptychographic microscopy using a generalized Anscombe transform approximation of the mixed Poisson-Gaussian likelihood,"
{\em Opt. Exp.} {\bf 25} (2017) 168-179. 

 \bibitem{FPM13}
 G. Zheng, R. Horstmeyer and C.Yang, ``Wide-field, high-resolution Fourier ptychographic microscopy," {\em Nature Photonics} {\bf 7} (2013), 739-745.
 
 \bibitem{adaptive}
C. Zuo, J. Sun and Q. Chen, ``Adaptive step-size strategy for noise-robust Fourier ptychographic microscopy," {\em Optics Express} {\bf 24} (2016), 20724-20744.
 
\end{thebibliography}
\end{document}